\documentclass{article}

\usepackage[final,nonatbib]{neurips_2021}

\usepackage[utf8]{inputenc} 
\usepackage[T1]{fontenc}    
\usepackage{hyperref}       
\usepackage{url}            
\usepackage{booktabs}       
\usepackage{amsfonts}       
\usepackage{nicefrac}       
\usepackage{microtype}      
\usepackage{xcolor}

\usepackage{romannum}
\usepackage[inline]{enumitem}
\usepackage{bm}
\usepackage{amsthm}
\usepackage{amsmath}
\usepackage{graphicx}
\usepackage{subcaption}
\usepackage{algorithmic}
\usepackage{algorithm}

\usepackage{float}

\newcommand{\comment}[1]{}
\renewcommand{\b}[1]{\boldsymbol{#1}}
\newcommand{\ck}{c_{\mathrm{sigmoid},k}}
\newcommand{\norm}[1]{\left\lVert#1\right\rVert}
\newcommand{\REAL}{\ensuremath{\mathbb{R}}}
\newcommand{\eps}{\varepsilon}
\newcommand{\br}[1]{\left\{#1\right\}} 
\newtheorem{theorem}{Theorem}
\newtheorem{definition}[theorem]{Definition}
\newtheorem{lemma}[theorem]{Lemma}

\DeclareMathOperator*{\argmax}{arg\,max}

\graphicspath{{./figures/new/sigmoid/}, {./figures/new/Cifar10/}, {./figures/logistic/}}
\begin{document}
	\pagenumbering{arabic}
	\title{Generic Coreset for Scalable Learning of Monotonic Kernels: Logistic Regression, Sigmoid and more}
\author{
  Elad Tolochinsky \\
  University of Haifa, Israel\\
  \texttt{eladt26@gmail.com} \\
   \And
   Ibrahim Jubran \\
   University of Haifa, Israel\\
   \texttt{ibrahim.jub@gmail.com}\\
   \And
   Dan Feldman \\
   University of Haifa, Israel\\
   \texttt{dannyf.post@gmail.com}
}
	\maketitle{}

\newif\ifproofs
\proofstrue
	
\begin{abstract}
	Coreset (or core-set) is a small weighted \emph{subset} $Q$ of an input set $P$ with respect to a given \emph{monotonic} function $f:\mathbb{R}\to\mathbb{R}$ that \emph{provably} approximates its fitting loss $\sum_{p\in P}f(p\cdot x)$ to \emph{any} given $x\in\mathbb{R}^d$. Using $Q$ we can obtain approximation of $x^*$ that minimizes this loss, by running \emph{existing} optimization algorithms on $Q$. In this work we provide: (i) A lower bound which proves that there are sets with no coresets smaller than $n=|P|$ for general monotonic loss functions. (ii) A proof that, under a natural assumption that holds e.g. for logistic regression and the sigmoid activation functions, a small coreset exists for \emph{any} input $P$. (iii) A generic coreset construction algorithm that computes such a small coreset $Q$ in $O(nd+n\log n)$ time, and (iv) Experimental results with open-source code which demonstrate that our coresets are effective and are much smaller in practice than predicted in theory.
\end{abstract}

\section{Introduction}
Traditional algorithms in computer science and machine learning are usually tailored to handle off-line finite datasets that are stored in memory. However, many modern systems do not use this computational model.
For example, GPS data from millions of smartphones, high definition images, YouTube videos, Twitter tweets, or audio signals from smart homes arrive in a streaming fashion. The era of Internet of Things (IoT) provides us with wearable devices and mini-computers that collect data sets that are being gathered by ubiquitous information-sensing mobile devices and wireless sensor networks~\cite{Hellerstein, Segaran,soda13}. 

\textbf{Challenges. }Using such devices and networks pose a series of challenges:\\
\textbf{(i) Limited memory. }In such systems, the input is an infinite stream of batches that may grow in practice to petabytes of raw data, and cannot be stored in memory. Hence, only one-pass over the data and small memory are allowed.\\ \textbf{(ii) Parallel computations. }To leverage the power of multithreading and multiple processing units (as in GPUs), we are required to design variants of our algorithms which can run in parallel.\\
\textbf{(iii) Distributed computations. } If the dataset is distributed among many machines, e.g. on a ``cloud", there is an additional problem of non-shared memory, which may be replaced by expensive and slow communication between the machines. 

\textbf{Weak or no theoretical guarantees. }
Due to the modern computation models above, learning trivial properties of the data may become non trivial, as stated in~\cite{soda13}. These problems are especially common in machine learning applications, where the common optimization problems and models may be, already in the off-line settings NP-hard. The result is neglecting, in some sense, decades of theoretical computer science research, and replacing it by fast heuristics and ad-hoc rules, which are easy to implement under the above constrains, and provide reasonable results. Those heuristics, however, have no theoretical guarantees, either as of running time or of global optimality.

\subsection{Coresets}
Coresets suggest a natural solution or at least a very generic approach to address the above challenges without re-inventing computer science. Coresets have some promising theoretical guarantees, while still leveraging the success of existing heuristics.
Instead of designing, from scratch, a new algorithm to solve the problem at hand, the idea is to provably summarize the data into a small representative subset, and to prove that applying \emph{existing} algorithms, both heuristics and provable methods, on this small summarizations, will yield an output which approximates the result of running the same algorithms on the original (full) data.

In this paper we focus on coresets for monotonic continuous functions, that is: we assume that we are given a set $P$ of $n$ points in $\REAL^d$, and a non-decreasing monotonic functions $f:\REAL\to\REAL_{>0}$. 
For a given error parameter $\eps\in (0,1)$, we wish to compute an \emph{$\eps$-coreset} $Q\subseteq P$, with a weight function $u:Q\to[0,\infty)$ that \emph{provably} approximates the fitting cost of $P$ for every $x\in\REAL^d$, up to a multiplicative factor of $1\pm\eps$, i.e., $(1-\eps)\sum_{p\in P}f(p\cdot x)\leq \sum_{p\in Q}w(p)f(p\cdot x)\leq (1+\eps)\sum_{p\in P}f(p\cdot x)$.
Although it seems rather theoretic, many real world problems can be formulated using such functions, including the Sigmoid, Logistic regression, SVM, Linear classifiers, and Gaussian Mixture Models; see examples in~\cite{feldman2020core}.

\textbf{Coresets and machine learning.}
We can use the notion of coresets as described above for improving the performance of machine learning algorithms. Most machine learning algorithms essentially solve an optimization problem over some set of training data. By constructing a coreset for this training data, we can: (i) greatly reduce the time it takes to train a model, simply by training it on the (small) coreset, and (ii) allow support for streaming, parallel, and distributed data. 
Although the coreset provides guarantees for the approximation of the MSE of the training data, it can be shown that for some problems, a coreset can also provide guarantee for the approximations of the generalizations error. For example, when using Bayesian inference, it was shown in~\cite{huggins2016coresetsLogistic} that a model which is based on coreset for the log likelihood function, has a marginal likelihood which is \emph{guaranteed} to approximate the true marginal likelihood. The same can be shown for maximum likelihood estimation. The popular measure for the goodness of fit of an estimator is the the log-likelihood ratio:  $\ln\Lambda(\hat{\theta}) = \mathcal{L}(\hat{\theta}) - \sup_{\theta \in \Theta} \mathcal{L}(\theta)$. The log-likelihood ratio of a model which is based on a coreset, uniformly approximates the log-likelihood ratio of the full model. 
Furthermore, coresets have been shown to practically improve the generalization error for machine learning algorithms~\cite{huggins2016coresetsLogistic, feldman2011scalable, munteanu2018coresets}

\subsection{Our contribution}

\textbf{(i) }We provide an impossibility bound that proves that, for non-decreasing monotonic loss function, there are no small coresets in general. We do this by providing an example of an input set of points $P$, for which no coreset of size smaller than $|P|$ exists; see Section~\ref{sec:lower_bounds}.

\textbf{(ii) }Following the bound above, we can either give up on the generic coreset paradigm, or add natural assumptions and modifications to the targeted functions $f$. In this paper we choose the second option;
We add a regularization term to the loss function, which, in most cases, is added anyway to avoid overfitting~\cite{scholkopf2002learning,bishop1995training}.
In fact, in some cases, this new term is crucial as some functions are minimized only for $x$ approaching infinity if this term is omitted. For example, the regularization term we add to the sigmoid function is $\norm{x}_2^2/k$, where $k>0$ defines the trade-off between minimizing the function and the complexity of the set of parameters.
While minimizing such functions may still be NP-hard~\cite{vsima2002training}, we prove that a small coreset $Q$ exists for \emph{any} input set $P$, for the sigmoid and logistic regression functions; see Section~\ref{sec:apps}. However, the proof holds for a wider family of functions.

 \textbf{(iii) }We provide a generic algorithm that computes the coreset $Q$ above in $O(nd+n\log n)$ time. 
Unlike most existing works, our algorithm can construct a coreset for the sigmoid and logistic regression functions, as well as a wider set of functions; see Algorithm~\ref{algno}. 

\textbf{(iv) }Open source code for our algorithms is given~\cite{opencode}, along with extensive experimental results on both synthetic and real-world public datasets; see Section~\ref{sec:ER}. 

\subsection{Related Work}\label{sec:related_work}
In \cite{har2006coresets}, Har-Peled shows how to construct a coreset
of one dimensional points sets $(d=1)$ for sums of single variable real valued
functions. In the scope of machine learning most of the research involves
clustering techniques~\cite{feldman2013turning,feldman2012data,feldman2007ptas}
and regressions~\cite{boutsidis2013nearOptimalLeastSquare,dasgupta2009samplingLpReg,zheng2017coresetsKR}.
Several coresets were constructed for unsupervised learning problems
including coresets for Gaussian mixture models~\cite{feldman2011scalable},
and SVM~\cite{tsang2005coreVectorMachine,har2007maximumSvm}. Other works handle general families of supervised learning problems~\cite{tukan2020coresets, maalouf2019fast}.

The work by~\cite{huggins2016coresetsLogistic} introduces lower bounds on the total sensitivity of the logistic regression problem that is used in this paper. It also introduces an upper bound for the total sensitivity and coreset size based on $k$-clustering coresets. However the bounds hold only for input set $P$ from very specific distributions (roughly, when $P$ is well separated into $k$ clusters).

In~\cite{munteanu2018coresets}, a lower bound of $\Omega\left(n/\log n\right)$ points, on the size of a coreset for a two dimensional logistic regression was introduced. To find a coreset, the authors have introduced a measure of the data $\mu$, which depends on the log-ratio between the positive and negative labeled points, and have shown that for data sets in which $\mu$ is sufficiently small a coreset of size $O(poly(\log n))$ exist. 
Instead of imposing assumptions on the above input-related measure, in this work we add a regularization term to the loss function which, as we show, makes the coreset construction task feasible. 
There does not seem to be a direct relation between our work and the measure $\mu$ used in~\cite{munteanu2018coresets}.

The main tool of this work uses the unified framework presented in~\cite{feldman2011unified}, which was recently improved in~\cite{braverman2016new}. We also use the reduction from $\mathcal{L}_\infty$ coresets that approximates $\max_{p\in P} f(p\cdot x)$ to our $\mathcal{L}_1$ coreset (sum of loss) which was introduced in~\cite{varadarajan2012near}.

\subsection{Paper Organization}
Section~\ref{sec:pre} describes preliminary results which we utilize in our coreset construction algorithm. 
In Section~\ref{sec:lower_bounds} we give examples of input sets which have no non-trivial coreset (i.e., smaller than the input size), for general monotonic functions. 
In Section~\ref{sec:coreset_for_mono} we introduce our main coreset construction algorithm. We then prove the correctness of this algorithm for the sigmoid and logistic regression activation functions.
In Section~\ref{sec:ER} we provide our experimental results along with a discussion.

\section{Preliminaries} \label{sec:pre}
In what follows we first describe the coreset construction framework of~\cite{feldman2011unified}. The framework is based on a non-uniform sampling of the input, according to some importance distribution over the input points. This distribution assigns higher values to points of higher influence
on the optimization problem at hand. 
Now, in order to keep the sample unbiased, the sampled points are reweighted reciprocal to their sampling probability. 
To quantify the influence of a single point on the optimization problem, Feldman and Langberg suggested in~\cite{LS10} a term called~\emph{sensitivity}, which we define later in this section.
Using the sensitivity, a sampling-based coreset can be constructed, whose size depends on the total sensitivity over the input points, a complexity measure of the family of models, called the VC-dimension, and an error parameter $\varepsilon \in (0,1)$ that controls the trade-off between coreset size and approximation accuracy.
Bounding the VC-dimension of the loss functions handled in this paper is straightforward; see formal details in Section~\ref{sec:VCBound} at the supplementary material. Hence, the majority of the paper is devoted to bound the sensitivity of each point.

We now formally define the sensitivity of every input point, with respect to a given problem at hand.
\begin{definition}[Sensitivity~\cite{feldman2011unified,LS10}]\label{def:sensitivity}
Let $(P,w,X,c)$ be a tuple called \emph{query space}, where $P$ is a finite set of elements, $w:P\to [0,\infty)$ is a weight function, $X$ is a set called \emph{queries} (models), and $c:P\times X \to [0,\infty)$ is a loss function.
The \emph{sensitivity} of a point $\b{p}\in P$ with respect to $(P,w,X,c)$ is defined as
\[
s(p):=s_{P,w,X,c}\left(\b{p}\right)=\sup_{\b{x}\in X}\frac{w\left(\b{p}\right)c\left(\b{p},\b{x}\right)}{\sum_{\b{p'}\in P}w\left(\b{p'}\right)c\left(\b{p'},\b{x}\right)},
\]
where the supremum is over every $\mathbf{x}\in X$ such that the denominator is positive .
The \emph{total sensitivity} of the query space is denoted by $t(P):=t(P,w,X,c)=\sum_{\b{p}\in P}s(\b{p})$.
\end{definition}

One of the contributions of~\cite{feldman2011unified} is to establish a connection to the theory of range spaces and the well known VC-dimension. Informally, the (VC) dimension of a given problem is a measure of its combinatorial complexity~\cite{anthony2009neural}. For completeness, a formal definition is given in the supplementary material; see Section~\ref{sec:VCBound}.

Feldman and Langberg also show how to compute, without further assumptions, a small weighted set $(Q,u)$, where $Q\subseteq P$, that will approximate the total cost $C\left(P,w,\b{x}\right)$ of the input $(P,w)$, for every query $\b{x} \in X$, up to a multiplicative factor of $1\pm\varepsilon$. Such a set, which we call a \emph{coreset}, is defined as follows.
\begin{definition}	[$\varepsilon$-coreset\label{coresetdef}] 
Let $\left(P,w,X,c\right)$ be a query space (see Definition~\ref{def:sensitivity}), and $\eps\in(0,1)$ be an error parameter. An \emph{$\varepsilon$}-\emph{coreset} for $\left(P,w,X,c\right)$ is a weighted set $\left(Q,u\right)$ such that for every $\b{x} \in X$,
\[
\left|\sum_{\b{p}\in P}w\left(\b{p}\right)c\left(\b{p},\b{x}\right)-\sum_{\b{q}\in Q}u\left(\b{q}\right)c\left(\b{q},\b{x}\right)\right| \leq\varepsilon \cdot \sum_{\b{p}\in P}w(\b{p})c(\b{p},\b{x}).
\]
\end{definition}
In~\cite{feldman2011unified}, a lower bound is given for the required coreset size, as a function of the total sensitivity $t(P)$. This bound was later made tighter in~\cite{braverman2016new}. The following theorem describes the random sampling scheme for coreset construction using the sensitivity framework, and describes the required sample (coreset) size.
\begin{theorem}[coreset construction~\cite{braverman2016new,feldman2011unified}]\label{theorem:sens_is_coreset} Let $\left(P,w,X,c\right)$
be a query space of VC-dimension $d$ and total sensitivity $t$. Let $\eps,\delta\in(0,1)$. Let $Q$ be a random sample of $\left|Q\right|\ge\frac{10t}{\varepsilon^{2}}\left(d\log t+\log\left(\frac{1}{\delta}\right)\right)$
i.i.d points from $P$, such that every $\b{p}\in P$ is sampled with probability $\frac{1}{t}\cdot s_{P,w,X,c}\left(\b{p}\right)$.
Let $u\left(\b{p}\right)=\frac{t\cdot w\left(\b{p}\right)}{s_{P,w,X,c}\left(\b{p}\right)\left|Q\right|}$
for every $\b{p}\in Q$. Then, with probability at least
$1-\delta$, $\left(Q,u\right)$ is an $\varepsilon$-coreset of $\left(P,w,X,c\right)$.
\end{theorem}

\section{Lower Bounds}\label{sec:lower_bounds}

In what follows, we consider query spaces $(P,w,\REAL^d,c)$, where $c(\b{x},\b{p}) = f(\b{x} \cdot \b{p})$ for some non-decreasing monotonic function $f$. We prove that not all such query spaces admit a non-trivial coreset, by providing an example of an input sets $P$ for which every coreset must be of size $|P|$.


\paragraph{No coreset.}
Consider a 2-dimensional circle $C \subseteq \REAL^3$ in $3$-dimensional space, which is the intersection of the unit sphere and a non-affine plane (does not pass through the origin) that is parallel to the $XY$ plane.
For every point $\b{p} \in P$, let $\pi_{\b{p}}$ be a plane in $\REAL^3$ that passes through the origin which isolates $\b{p}$ from the rest of the set, and let $\b{x_p}$ be a vector orthogonal to $\pi_{\b{p}}$, such that $\b{p} \cdot \b{x_p} > 0$; see Fig~\ref{fig:P}. Such a plane exists since the points are on a 2D circle that is not centered around the origin. 

Now, for intuition, consider the logistic regression cost function: $c(\b{x}, \b{p}) = \log(1+e^{\b{p} \cdot \b{x}})$. Let $\b{p} \in P$ and let $\b{x_p}$ be the query vector orthogonal to the plane $\pi_{\b{p}}$ which separates $\b{p}$  from the rest of the set. Since $\b{p}$ is the only point on the positive side of $\b{x_p}$, it holds that $\b{p}\cdot \b{x_p} > 0$ whereas for every other point $\b{p'}$, $\b{p'}\cdot \b{x_p} < 0$. Moreover as $\norm{\b{x_p}}$ grows, $\b{p}\cdot \b{x_p}$ goes to $\infty$ and $\b{p'}\cdot \b{x_p}$ grows to $-\infty$. Thus the cost $c(\b{x_p}, \b{p}) = \log(1+e^{\b{p} \cdot \b{x_p}})$ of $\b{p}$, goes to $\infty$ and the cost of every other point goes to $0$. Therefore, $\b{p}$ has a sensitivity of $1$. In this case, intuitively, every coreset must include $\b{p}$ or else it cannot provide a good approximation to the cost of the original (full) set. Since this argument holds for every $\b{p} \in P$, any coreset for $P$ must include all points in $P$. Thus, no non-trivial coreset exists in this case.
Putting, it differently, the above discussion shows that if the sensitivity of every point in $P$ is $1$ then the size of every coreset is $\Omega(n)$; see Lemma~\ref{lemma:sen_is_nec} in the supplementary material for a formal statement.


\begin{figure}[h]
    \centering
    \includegraphics[width=0.45\linewidth]{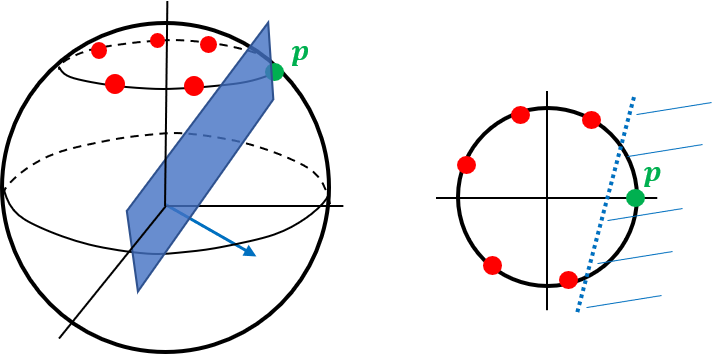}
    \caption{(Left): A set of points $P$ in $\REAL^3$ (red and green points), a plane $\pi_{\b{p}}$ separating $p$ from $P\setminus\br{p}$ and the vector $\b{x_p}$ orthogonal to $\pi_{\b{p}}$. (Right): A top-down view of the data on the left. The dotted line $\ell$ is the intersection of $\pi_{\b{p}}$ and the plane containing $P$. All the points to the right (left) of $\ell$ are projected onto the positive (negative) side of $\b{x_p}$.}
    \label{fig:P}
\end{figure}

Note that the above holds true not only for logistic regression but for any function $f$ that satisfies $\lim_{x\rightarrow\infty}\frac{f\left(-x\right)}{f\left(x\right)}=0$. This is formally stated in the following theorem. A formal proof is given in Section~\ref{sec:noCoreset} of the supplementary material.
\begin{theorem}\label{theorem:No coreset}
Let $f:\REAL\rightarrow\left(0,\infty\right)$
be a non-decreasing monotonic function that satisfies $\lim_{x\rightarrow\infty}\frac{f\left(-x\right)}{f\left(x\right)}=0$,
and let $c\left(\b{x},\b{p}\right)=f\left(\b{x}\cdot\b{p}\right)$
for every $\b{x},\b{p}\in\REAL^{d}$. Let $\varepsilon\in\left(0,1\right)$, $n\geq 1$ be an integer, and $w:\REAL^d\rightarrow\left(0,\infty\right)$. There is a set $P\subset\REAL^{d}$ of $|P|=n$ points such that if $(Q,u)$ is an $\eps$-coreset of $\left(P,w,\REAL^{d},c\right)$ then $Q=P$.
\end{theorem}

\textbf{Adding assumptions. }The above counter-example and formal claim motivate the necessity of adding assumptions on the loss function, as described in the following section. Mainly, a regularization term needs to be added. This term is usually added anyway, both in theory and in practice, to reduce the complexity of the model and avoid overfitting.

\section{Coresets For Monotonic Bounded Functions}\label{sec:coreset_for_mono}
From the previous section, we conclude that an additional constraint must be imposed on the problem at hand in order to construct a small coreset.
To better understand the required constraint, recall the reason for the lower bound from the example at Section~\ref{sec:lower_bounds}; the (problematic) points with sensitivity 1 were the points which had very large values of $\b{x}\cdot \b{p}$. This can happen when $\norm{\b{x}}$ is very large or when $\norm{\b{p}}$ is large. For the moment, assume that $\norm{\b{p}}$ is small (we will later see how $\norm{\b{p}}$ affects the size of the coreset). 
The standard technique for preventing the parameters from growing too large is to add a regularization term, which is widely used in many real world problems~\cite{scholkopf2002learning, kukavcka2017regularization}. As it happens to be, adding a regularization term also advances us towards our goal of constructing a coreset, as was also noted e.g., in~\cite{samadian2020unconditional, tukan2021coresets}. To see this, consider a regularized variant of the loss function: $c(\b{x}, \b{p}) = f\left(\b{x}\cdot\b{p}\right) + \frac{\norm{\b{x}}}{k}$. Since $f$ is bounded, when $\norm{\b{x}}$ grows to infinity the value of the regularization dominates the loss. Thus, in this case, all points have approximately the same loss, and are all equally unimportant. In other words, the sensitivity of those points can not be $1$.

The common case for the value of the regularization parameter $k$ is $k = n^{1-\kappa}$ for $\kappa \in (0,1)$; see e.g., in~\cite{curtin2019coresets,mai2021coresets}. In practice, we observed that the values of $k$ have only a small effect on the coresets approximation accuracy; see Section~\ref{sec:ER}.

\subsection{$\mathcal{L}_\infty$ coresets} \label{sec:L_inf_coreset}
We now address the common case, in which for some $\b{x} \in X$ and every two points, $\b{p}_1$, $\b{p}_2 \in P$ the values of $\b{p}_1 \cdot \b{x}$ and $\b{p}_2 \cdot \b{x}$ do not greatly differ. To do so, we will reduce our problem to the problem of constructing an $\mathcal{L}_\infty$ coreset, which is defined as follows.

\begin{definition}($\mathcal{L}_\infty$ coreset\cite{varadarajan2012near})
    Let (P,w,X,c) be a query space and $\eps>0$. An $\eps-\mathcal{L}_\infty$ coreset is a subset $Q \subseteq P$ such that 
    $\max_{\b{p}\in P} c(\b{p}, \b{x}) \le (1+\eps)\max_{\b{q}\in Q} c(\b{q}, \b{x})$ for every $\b{x} \in X$.
\end{definition}

We will now focus on constructing an $\mathcal{L}_\infty$ coreset. We will then show how to leverage this $\mathcal{L}_\infty$ coreset to obtain a coreset as defined in Definition~\ref{coresetdef}.

Consider a monotonic non-decreasing function $f:\REAL \to (0,M]$, a query $\b{x} \in X$ and a point $\b{p} \in P$ such that $\b{p} \cdot \b{x} > 0$. Since $f$ is a monotonic function, $f(0) \leq f(\b{p} \cdot \b{x})$. Hence, 	
\[
    \max_{\b{p}'\in P}f(\b{p}' \cdot \b{x})\le M = \frac{M}{f(0)}f(0) \le \frac{M}{f(0)}f(\b{p}\cdot \b{x}),\label{eq:case1}
\]

Therefore, for a query $\b{x}$, if a point $\b{p}$ falls on the positive side of the line defined by $\b{x}$ we can say this point is an $\mathcal{L}_\infty$ coreset. But what if the point falls on the negative side of the line? Since $f$ is monotonic, we know that if $\b{p} \cdot \b{x} <0$ then, $f(\b{p} \cdot \b{x}) < f(-\b{p} \cdot \b{x})$, but if $f$ is sufficiently ``well behaved'' then as long as the distance between $-\b{p} \cdot \b{x}$ and $\b{p} \cdot \b{x}$ is not too large, then the distance between $f(-\b{p} \cdot \b{x})$ and $f(\b{p} \cdot \b{x})$ is also bounded. Specifically, we can assume there is a constant $b > 0$ such that 
\[
    f(-\b{p} \cdot \b{x}) < b\cdot f(\b{p} \cdot \b{x})
\]
which implies that even if $\b{p}$ falls on the negative side of the line, then $\b{p}$ is an $\mathcal{L}_\infty$ coreset.  

\textbf{Assumptions and conclusions made so far. }Before we conclude the results of this section, we must conduct the assumptions and conclusions we have made so far. We have assumed that the distance between $-\b{p} \cdot \b{x}$ and $\b{p} \cdot \b{x}$ is not too large. We can bound the distance as follows:
\[
    |-\b{p} \cdot \b{x} - (\b{p} \cdot \b{x})| = |2\b{p} \cdot \b{x}| \le 2\norm{\b{x}}\norm{\b{p}}.
\]
From the discussion in the beginning of the section, adding regularization will guarantee that $\norm{\b{x}}$ can not grow arbitrarily large.
As for the $\norm{\b{p}}$ term, we expect the coreset to be somehow affected by this term in order to ensure the above property. Indeed, this is one of the main terms which affect the sensitivity of the input points. Hence, the final coreset will be more likely to choose points with larger norm.

We conclude that every point $\b{p} \in P$ is an $\eps-\mathcal{L}_\infty$ coreset for sufficiently large $\eps$ that depends on properties of the function ($M, f(0)$ and $b$) and on $\norm{\b{p}}$. This is formally stated in the following lemma.
\begin{lemma} [$\mathcal{L}_\infty$ coresets] \label{lemma:exsitance of inf bounded}
Let $P\subset\REAL^{d}$ be a finite set, $M, k>0$ be constants, $f:\REAL\rightarrow(0,M]$
be non-decreasing function and $g:[0,\infty)\rightarrow[0,\infty)$ be a function. For every $\b{x}\in\REAL^{d}$ and $\b{p}\in P$ define $c_{k}\left(\b{p,x}\right)=f\left(\b{p}\cdot\b{x}\right)+\frac{g\left(\left\Vert \b{x}\right\Vert \right)}{k}$.
Put $\b{p}\in P$ and suppose there is $b_{\b{p}} > 0$ such that for every $z>0$, $f\left(\left\Vert \b{p}\right\Vert z\right)+\frac{g\left(z\right)}{k}\le b_{\b{p}}\left(f\left(-\left\Vert \b{p}\right\Vert z\right)+\frac{g\left(z\right)}{k}\right)$.
Then $\br{\b{p}}$ is an $\eps-\mathcal{L}_\infty$ coreset with $\eps = \frac{M}{f\left(0\right)}\left(b_{\b{p}}+1\right)-1$.
\end{lemma}

\newcommand{\alg}{\textsc{Monotonic-Coreset}}

\begin{algorithm}[]
	\begin{algorithmic}[1]
		\STATE {\bfseries Input:} A set $P=\br{p_1,\cdots,p_n}$ of points in $\REAL^d$, \\
        \quad\quad\quad a real valued regularization term $k>0$, and \\
        \quad\quad\quad an integer $m \geq 1$.
		\STATE {\bfseries Output:} A pair $(Q,u)$ where $|Q| = m$ and $u:Q \to [0,\infty)$; see Theorems~\ref{theorem:sigmoid}-\ref{theorem:logistic}.
		
		\STATE Sort the points in $P=\br{\b{p}_1,\cdots,\b{p}_n}$ by their length, i.e., $\norm{\b{p}_1}\leq\cdots\leq \norm{\b{p}_n}$.
		
		\STATE $\displaystyle s(\b{p}_i) := \frac{c\cdot \sqrt{k}\norm{\b{p}_i}+2}{i}$ for every $i \in [n]$ \COMMENT{$c$ is a sufficiently large constant.}
		
        \STATE Set $t\gets \sum_{i=1}^n s(\b{p}_i)$
		
		\STATE Pick an i.i.d random sample $Q \subseteq P$ of $|Q|\geq \min\br{m,n}$ from $P$, where every $\b{p} \in P$ is chosen with probability $s(\b{p})/t$.

		\STATE $\displaystyle u\left( \b{p} \right) := \frac{1}{\left|Q\right| \mathrm{Prob} \left( \b{p} \right)}$ for every $\b{p} \in Q$

		\STATE $\textbf{return} \left(Q, u \right)$
	\end{algorithmic}
	
	\caption{\alg($P,k,m$)\label{algno}}
\end{algorithm}

\subsection{From $\mathcal{L}_\infty$ coresets to coresets} \label{sec:fromLInfToCoresets}
We now describe how to leverage an $\mathcal{L}_\infty$ coreset to bound the sensitivity of every input point.

\textbf{Intuition behind Algorithm~\ref{algno}. }Let $Q$ be an $\eps-\mathcal{L}_\infty$ coreset of $P$. Intuitively, since the points in $Q$ provide a $(1+\eps)$-approximation to the maximal cost, we would require a random sampling scheme to choose these points with relatively high probability (compared to points in $P\setminus Q$). Let $Q_2$ be an $\eps-\mathcal{L}_\infty$ coreset of $P\setminus Q$. Using the same reasoning, we would require the probability of sampling points in $Q_2$ to be greater then the probability of sampling a point in $P\setminus Q \setminus Q_2$, but less than the probability of sampling a point in $Q$. Using this logic, we can continue to construct $\mathcal{L}_\infty$ coresets and remove them from the set of remaining points. The probability of every point $\b{p} \in P$ should intuitively be proportional to $\frac{1}{i}$, where $i$ is the index of the $\mathcal{L}_\infty$ coreset which contains $\b{p}$. Phrasing this differently: for every $\b{p}$, the sensitivity of $\b{p}$ is proportional to $\frac{1}{i}$. In~\cite{varadarajan2012near} it was proven that by repeatedly constructing $\mathcal{L}_\infty$ coresets as described above, one can bound the total sensitivity and construct a coreset. Fig.~\ref{fig:L_infty} illustrates the above reduction.

The following lemma gives the formal statement for the algorithm described above. The lemma is based on Lemma 3.1 in~\cite{varadarajan2012near}. 
\begin{lemma} \label{lem:reduction}
Let $c:\REAL^d \times \REAL^d \to (0, \infty)$. Suppose that for some $\eps \in(0,1)$ there is a non-decreasing function $\Delta_{\eps}(n)$ so that for any $P' \subseteq \REAL^d$ of size $n$ there is an $\eps-\mathcal{L}_\infty$ coreset of size at most $\Delta_{\eps}(n)$ for $(P', \mathbf{1}, \REAL^d, c)$. Then, for any $P \subseteq \REAL^d$ of size $n$ we can compute an upper bound $s(p)$ on the sensitivity $s_{P,\mathbf{1},\REAL^d,c}(p)$ for each $p\in P$, so that $\sum_{p\in P}s_{P,\mathbf{1},\REAL^d,c}(p) \leq (1+\eps)\Delta_{\eps}(n)\ln{n}$. 
\end{lemma}

\begin{figure}[h]
    \centering
    \includegraphics[width=\linewidth]{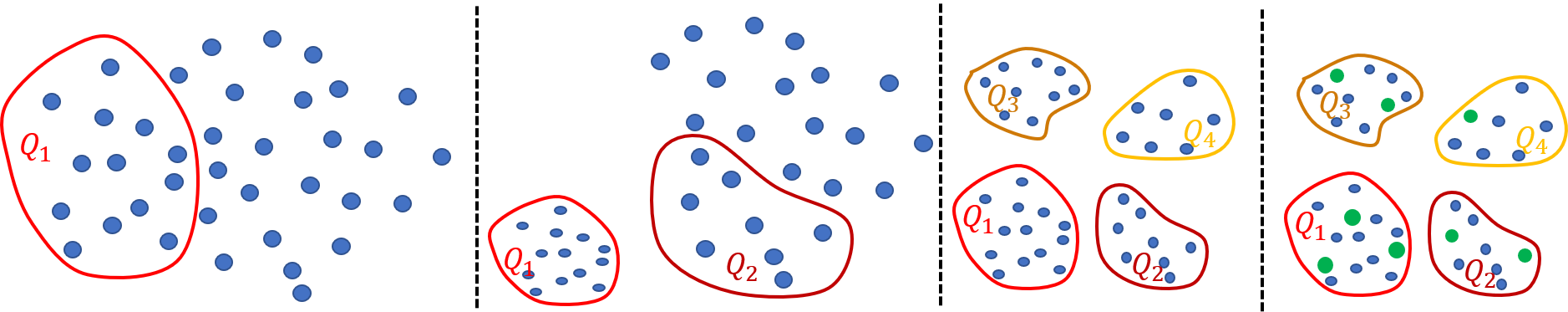}
    \caption{From left to right - \textbf{(i): }Construct an $\mathcal{L}_\infty$ coreset $Q_1$ from $P$. \textbf{(ii): }Remove $Q_1$ from $P$ and construct an $\mathcal{L}_\infty$ coreset $Q_2$ for $P \setminus Q_1$. \textbf{(iii): }Continue to do so until the sets $Q_1,Q_2,\cdots$ cover the entire set $P$. \textbf{(iv): }Sample a subset of $P$, such that every point $\b{p} \in P$ is sampled with probability proportional to $\frac{1}{i}$, where $\b{p} \in Q_i$. The resulting set, after a re-weighting reciprocal to its sampling probability, is a coreset.}
    \label{fig:L_infty}
\end{figure}

\textbf{A minor pitfall.} The algorithm described above, assumes that all of the $\mathcal{L}_\infty$ coresets have the same approximation constant $\eps$. However, this assumption does not hold in our case since the approximation constants of the $\mathcal{L}_\infty$ coresets we have constructed in the previous section depend on $\norm{\b{p}}$. Fortunately, we can still use the same general idea as before: in every iteration of the algorithm we have multiple choices to construct an $\mathcal{L}_\infty$ coreset, we must choose the correct order of construction so the total sensitivity will be the smallest. To understand this optimal order, we need to understand how the sensitivity of a point $s(\b{p})$ depends on the approximation constant $\eps$. As was shown in~\cite{varadarajan2012near}, $s(\b{p})$ linearly depends on $\eps$, or in our case $\norm{\b{p}}$, and on $1/i$ where is $i$ is the index of the point in some ordering. Thus for every point $\b{p}$, $s(\b{p})$ is proportional to $\norm{\b{p}}/i$. To minimize the total sensitivity we will prefer first to choose the points with smaller norms, so that the sensitivity of the points with the larger norms will be divided by a greater constant $i$. Algorithm~\ref{algno} gives a suggested implementation for the algorithm from the discussion above and the following theorem formally states the results.

\begin{theorem} \label{theorem:coreset}
Let $M, k>0$ be constants, $P \subseteq \REAL^d$ be a set of points, $f:\REAL\to (0,M]$ be a monotonic non-decreasing function, and $c_{k}\left(\b{p',x}\right)=f\left(\b{p'}\cdot\b{x}\right)+\frac{g(\norm{x})}{k}$ for every $\b{x}\in\REAL^{d}$ and $\b{p'}\in P$. 
Suppose there is $b:P \to (0,\infty)$ such that for every $\b{p}\in P$ and every $z>0$ we have $f\left(\norm{\b{p}} z\right)+\frac{g\left(z\right)}{k}\le b({\b{p}})\left(f\left(-\norm{\b{p}} z\right)+\frac{g\left(z\right)}{k}\right)$.
Let $b_{\max} \in \argmax_{\b{p} \in P} b({\b{p})}$, $t=(1+\frac{M}{f(0)} b_{\max})\ln n$, and $\eps,\delta\in(0,1)$. Lastly, let $d_{VC}$ be the VC-dimension of $(P,\mathbf{1},\REAL^d,c_k)$.
Then, there is a weighted set $(Q,u)$, where $Q\subseteq P$ and
$|Q| \in O\left(\frac{t}{\varepsilon^{2}}\left(d_{VC}\log t+\log\frac{1}{\delta}\right)\right)$,
such that with probability at least $1-\delta$, $(Q,u)$ is an $\eps$-coreset for the query space $(P,\mathbf{1},\REAL^d,c_k)$.
\end{theorem}

\textbf{Discussion behind Theorem~\ref{theorem:coreset}. }The above theorem suggests a sufficient condition for the existence of a coreset, in the case of a monotonic non-decreasing function $f$, to which a regularization term is added. 
The proof of this theorem is constructive; it combines the above condition with Lemma~\ref{lem:reduction} in order to bound the sensitivity of every input point $\b{p} \in P$ and also gives an upper bound to the total sensitivity; see Section~\ref{sec:coresetSuffCond} of the supplementary material. 
As an example, the following section constructs a coreset for the sigmoid and logistic regression activation functions by proving that the above condition is indeed met. However, the above theorem is not limited to those activation functions, and can be utilized for many other functions.
Given this sensitivity upper bound, the coreset construction algorithm is straightforward: it simply samples the input set $P$ based on the sensitivity distribution, and assigns appropriate weights to the sampled points. The only thing left to determine is the sample size required in order to achieve some predefined approximation error $\varepsilon$.
A suggested implementation for the sigmoid and logistic regression functions is given in Algorithm~\ref{algno}.

\section{Example Applications - Coresets for Sigmoid and Logistic Regression} \label{sec:apps}
In this section, we leverage the framework derived in the previous section in order to construct, as an example, a coreset for the sigmoid and logistic regression activations; see Theorems~\ref{theorem:sigmoid} and~\ref{theorem:logistic} respectively. The full proofs are placed in Section~\ref{sec:mainProofs} of the supplementary material.

\textbf{Overview of Theorems~\ref{theorem:sigmoid}-\ref{theorem:logistic}.} The following theorems construct a coreset for sums of sigmoid functions and for the logistic regression log-likelihood, for normalized input sets. To do so, we: (i) prove that the sufficient condition from Lemma~\ref{lemma:exsitance of inf bounded} and Theorem~\ref{theorem:coreset} in the previous section is met for both the sigmoid and the logistic regression functions; see Lemma~\ref{lemma:L2RegBoundSig_main} and Lemma~\ref{lemma:logistic} respectively. (ii) Based on the sufficient condition, we give an upper bound for the sensitivity of every input point as well as bound the total sensitivity; see Lemma~\ref{lem11} and Lemma~\ref{lem7} respectively. (ii) Lastly, we combine the above with the coreset construction framework from Theorem~\ref{theorem:sens_is_coreset} to obtain a provable sampling algorithm for coreset construction, as formally stated in  Theorems~\ref{theorem:sigmoid}-\ref{theorem:logistic}. An important ingredient in this construction was an upper bound for the VC-dimension of the relevant query spaces. An upper bound of $O(d^2)$ for both functions is given in Section~\ref{sec:VCBound}.
\begin{theorem} \label{theorem:sigmoid}
Let $P$ be a set of $n$ points in the unit ball of $\REAL^d$, $\eps,\delta\in(0,1)$, $k>0$ be a sufficiently large constant, and let $t=(1+k)\log n$.
For every $p,x\in\REAL^d$, let $\ck\left(\b{p},\b{x}\right)=\frac{1}{1+e^{-\b{p}\cdot\b{x}}}+\frac{\left\Vert \b{x}\right\Vert ^{2}}{k}$.
Finally, let $(Q,u)$ be the output of a call to $\alg(P,k,m)$, where $m \in \Omega\left(\frac{t}{\varepsilon^{2}}\left( d^2\ln{t} + \ln{\frac{1}{\delta}} \right)\right)$; see Algorithm~\ref{algno}.
Then, with probability at least $1-\delta$, $(Q,u)$ is an $\eps$-coreset for $(P,\mathbf{1},\REAL^d,\ck)$.
Moreover, $|Q|\in O(m)$, and $(Q,u)$ can be computed in $O(nd+n\log n)$ time.
\end{theorem}

\renewcommand{\ck}{c_{\mathrm{logistic},k}}

\begin{theorem} \label{theorem:logistic}
	Let $P$ be a set of $n$ points in the unit ball of $\REAL^d$, $\eps,\delta\in(0,1)$, $R,k>0$ where $k$ is a sufficiently large constant, and $t=R\log n(1+Rk)$.
	For every $\b{p}\in\REAL^d$,$\b{x}\in B(\b{0}, R)$ let $
	\ck(\b{p},\b{x})=\log\left(1+e^{\b{p}\cdot\b{x}}\right)+\frac{\norm{\b{x}}^{2}}{k}$.
	Finally, let $(Q,u)$ be the output of a call to $\alg(P,k,m)$ where $m \in \Omega\left(\frac{t}{\varepsilon^{2}}\left( d^2\ln{t} + \ln{\frac{1}{\delta}} \right)\right)$; see Algorithm~\ref{algno}. Then, with probability at least $1-\delta$, $(Q,u)$ is an $\eps$-coreset for $(P,\mathbf{1},\REAL^d,\ck)$.
	Moreover, $|Q|\in O(m)$ and $(Q,u)$ can be computed in $O(nd+n\log n)$ time.
\end{theorem}

\textbf{Supporting other activation functions. }The above theorems give two example activation functions that our framework supports. However, the framework is not limited to those activations only. To support other functions, one must prove the sufficient condition to obtain the sensitivity upper bound, which can be then simply plugged into Algorithm~\ref{algno} to obtain the desired coreset.

\section{Experiments}\label{sec:ER}
We implemented Algorithm~\ref{algno} and, in this section, we evaluate its empirical results both on synthetic and real-world datasets.
Rather than competing with existing solvers, our coreset is simply a pre-processing step which reduces the input size. To this end, we apply existing solvers as a black box on our small coreset. The results show that a coreset of size only $1\%$ of the original data can represent the full data with an error $\varepsilon$ smaller than $0.001$. Open-source code can be found in~\cite{opencode}.

\textbf{Competing methods. }Our main competing method is a random sampling scheme. As implied by the theoretical analysis, ``important'' points, i.e., with high sensitivity, are sampled with high probability in our coreset construction algorithm. However, such points are sampled with probability roughly $1/n$ using the naive uniform sampling. Hence, we expect the coreset would yield results much better than a uniform sampling scheme.
With that said, we chose real-world databases with relatively uniform data, in order to demonstrate the effectiveness of our coreset even in such cases. Even in this case, the improvement over uniform sampling is consistent and usually significant.




\textbf{Datasets used. }We used the following datasets:\\
\textbf{(i) Synthetic dataset.} This data contains a set of $n=20,010$ points in $\REAL^2$. $20,000$ of the points were generated by sampling a two dimensional normal distribution with mean $\mu_1=(10,000, 10,000)$ and covariance matrix $\Sigma_1 = \left(\begin{smallmatrix}0.0025 & 0\\ 0 & 0.0025\end{smallmatrix}\right)$ and $10$ points were generated by sampling a two dimensional normal distribution with mean $\mu_2=(-9998, -9998)$ and covariance matrix $\Sigma_2 = \left(\begin{smallmatrix}0.0025 & 0\\ 0 & 0.0025\end{smallmatrix}\right)$.\\
\textbf{(ii) Bank marketing dataset~\cite{moro2014data}. }It contains $n=20,000$ numerical valued records in $d=10$ dimensional space with.
The data was generated for direct marketing campaigns of a Portuguese banking institution. Each record represents a marketing call to a client, that aims to convince him/her to buy a product (bank term deposit). A binary label (yes or no) was added to each record. We used the numerical values of the records to predict if a subscription was made.\\
\textbf{(iii) Wine Quality dataset~\cite{cortez2009modeling,wang2013improving, elidan2010copula, kajino2012convex}. }It contains $n=6497$ numerical valued records in $d=12$ dimensional space. 

\textbf{Experiments. }We conducted the following experiments:\\
\textbf{(i) Sigmoid Activation.}
For a given size $m$ we computed a coreset of size $m$ using Algorithm~\ref{algno}. We used the datasets above to produce coresets of size $5\ln(n) \le m \le20\ln(n)$, where $n$ is the size of the full data, then we normalized the data and found the optimal solution to the problem with values of $k=100, 500, 1000, 5000$ using the BFGS algorithm. We repeated the experiment with a uniform sample of size $m$. For each optimal solution that we have found, we computed the sum of sigmoids and denoted these "approximated solutions" by $C_{1}$ and $C_{2}$ for our algorithm and uniform sampling respectively. The "ground truth" $C^{k}$ was computed using BFGS on the entire dataset. The empirical error is then defined to be $E_t=\left|\frac{C_t}{C^k}-1\right|$ for $t = 1,2$. For every size $m$ we computed $E_1$ and $E_2$ $100$ times and calculated the mean of the results.\\
\textbf{(ii) Logistic Regression.}
Similarly, we produced coresets and uniform samples of size $5\ln(n) \le m \le 40\ln(n)$ and maximized the regularized log-likelihood using the BFGS algorithm.
For every sample size we calculated the negative test log-likelihood. Every experiment was repeated 20 times and the results were averaged. 
All the results are presented in Fig.~\ref{fig:Comparison-between-uniform}.

\begin{figure}[h]
\centering
	\begin{subfigure}{.32\linewidth}
		\centering \includegraphics[width=\linewidth]{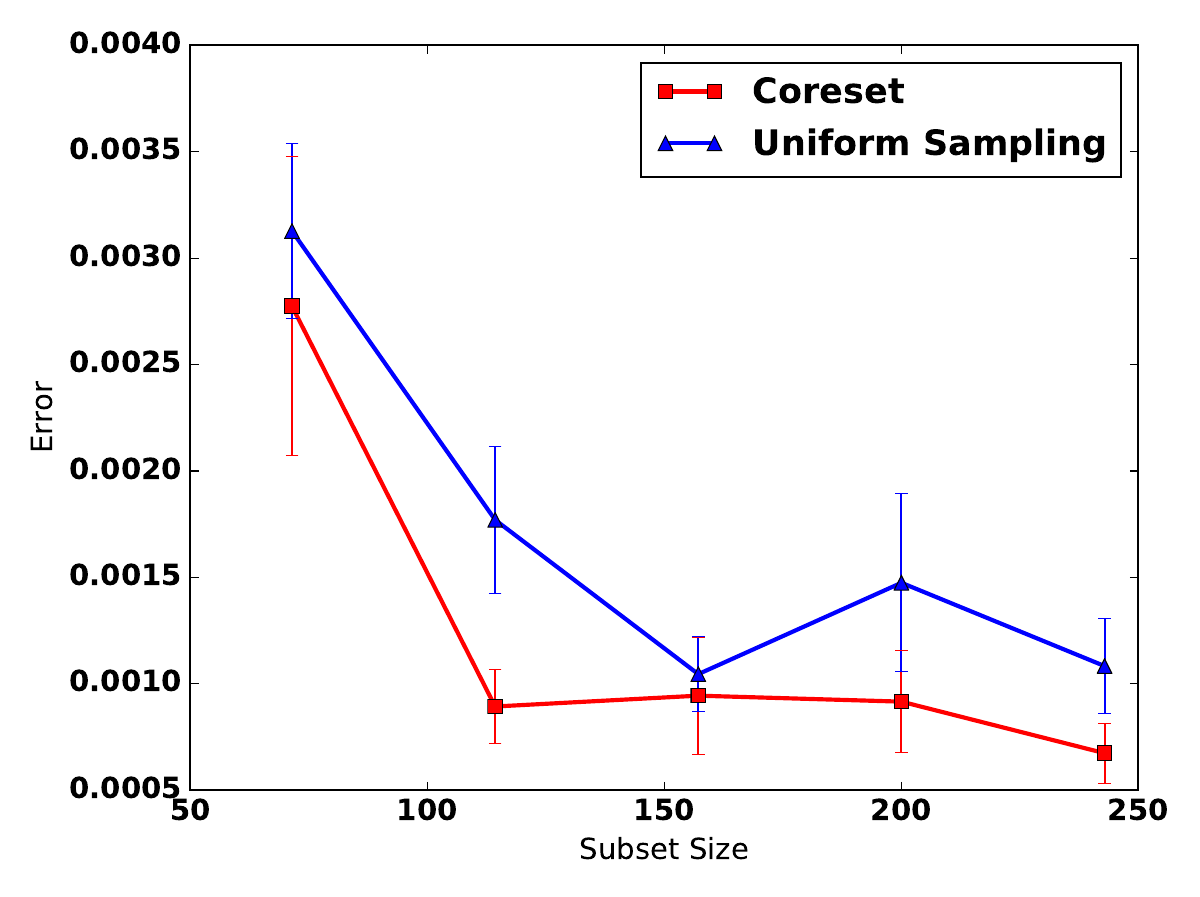}
		\caption{Bank Marketing dataset, $k=100$}
		\label{fig:a}
	\end{subfigure}
	\begin{subfigure}{.32\linewidth}
		\centering \includegraphics[width=\linewidth]{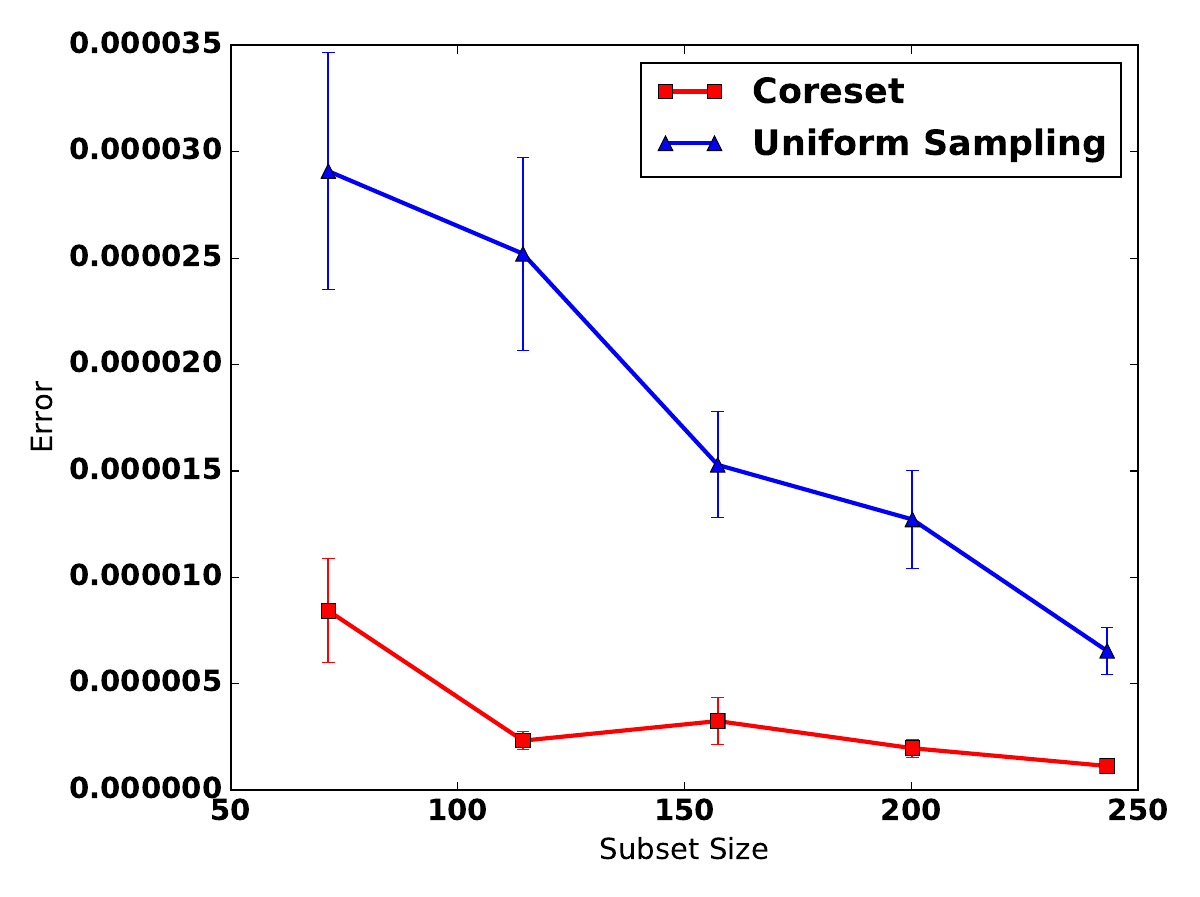}
		\caption{Synthetic dataset, $k = 500$}
		\label{fig:b}
	\end{subfigure}
	\begin{subfigure}{.32\linewidth}
		\centering \includegraphics[width=\linewidth]{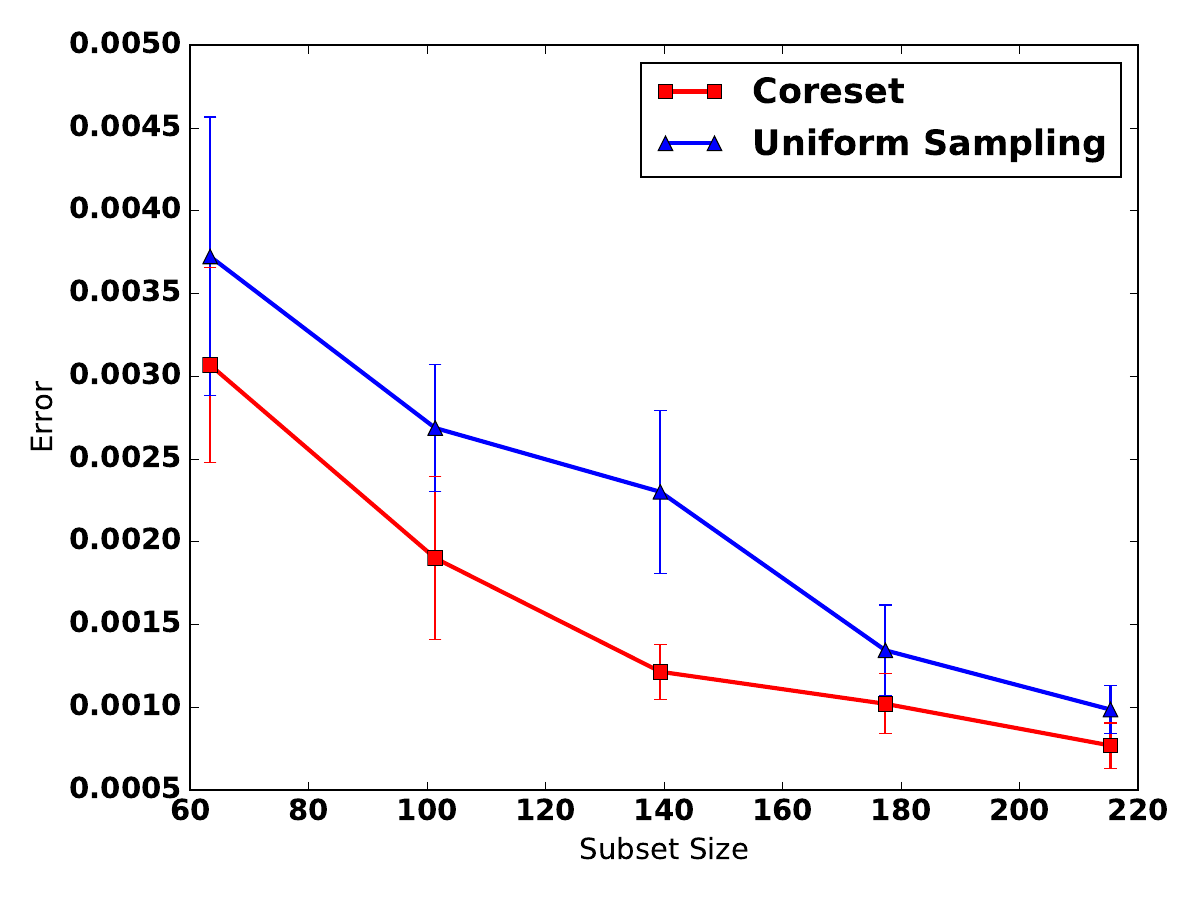}
		\caption{Wine dataset, $k=1000$}
		\label{fig:c}
	\end{subfigure}
	\vfill
	\begin{subfigure}{.4\linewidth}
		\centering \includegraphics[width=\linewidth]{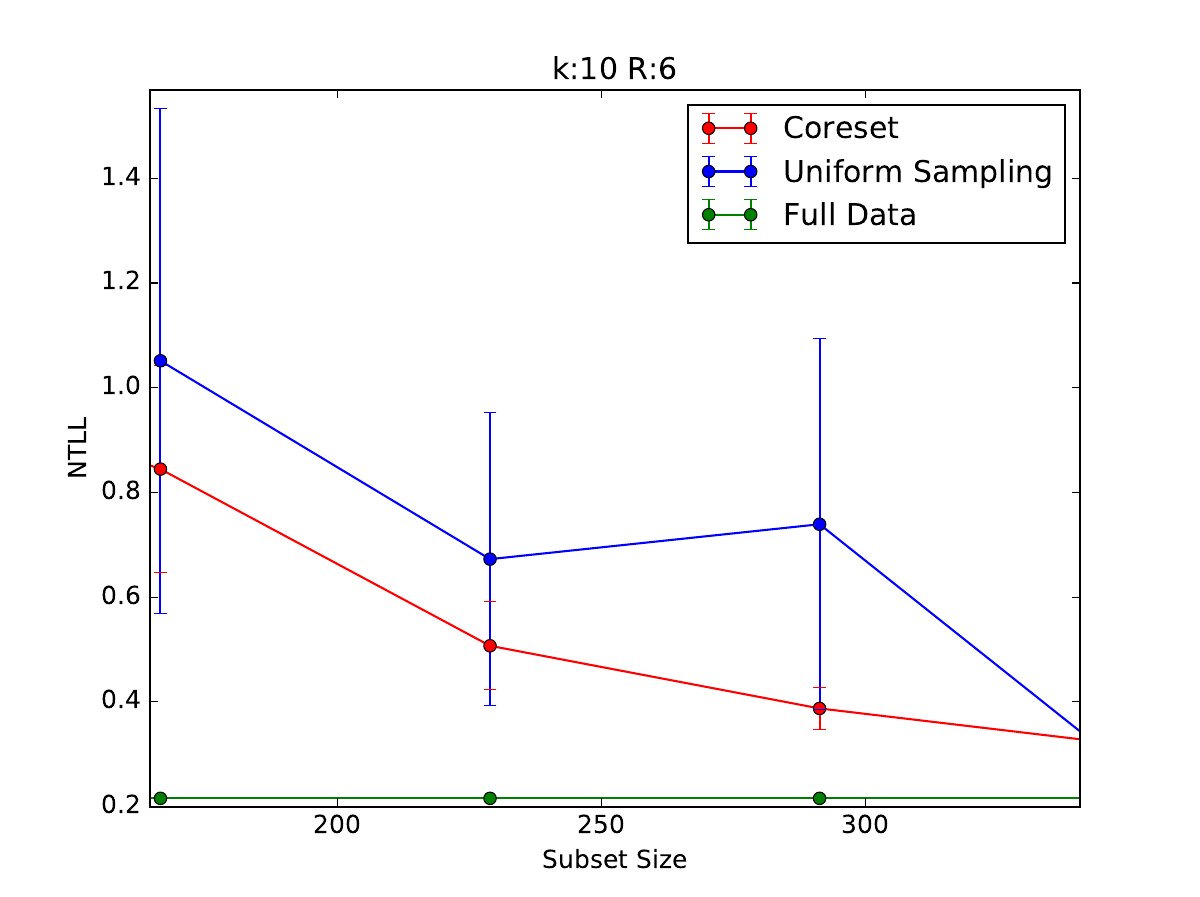}
		\caption{Bank Marketing dataset, $k=10,R=6$}
		\label{fig:d}
	\end{subfigure}
	\begin{subfigure}{.4\linewidth}
		\centering \includegraphics[width=\linewidth]{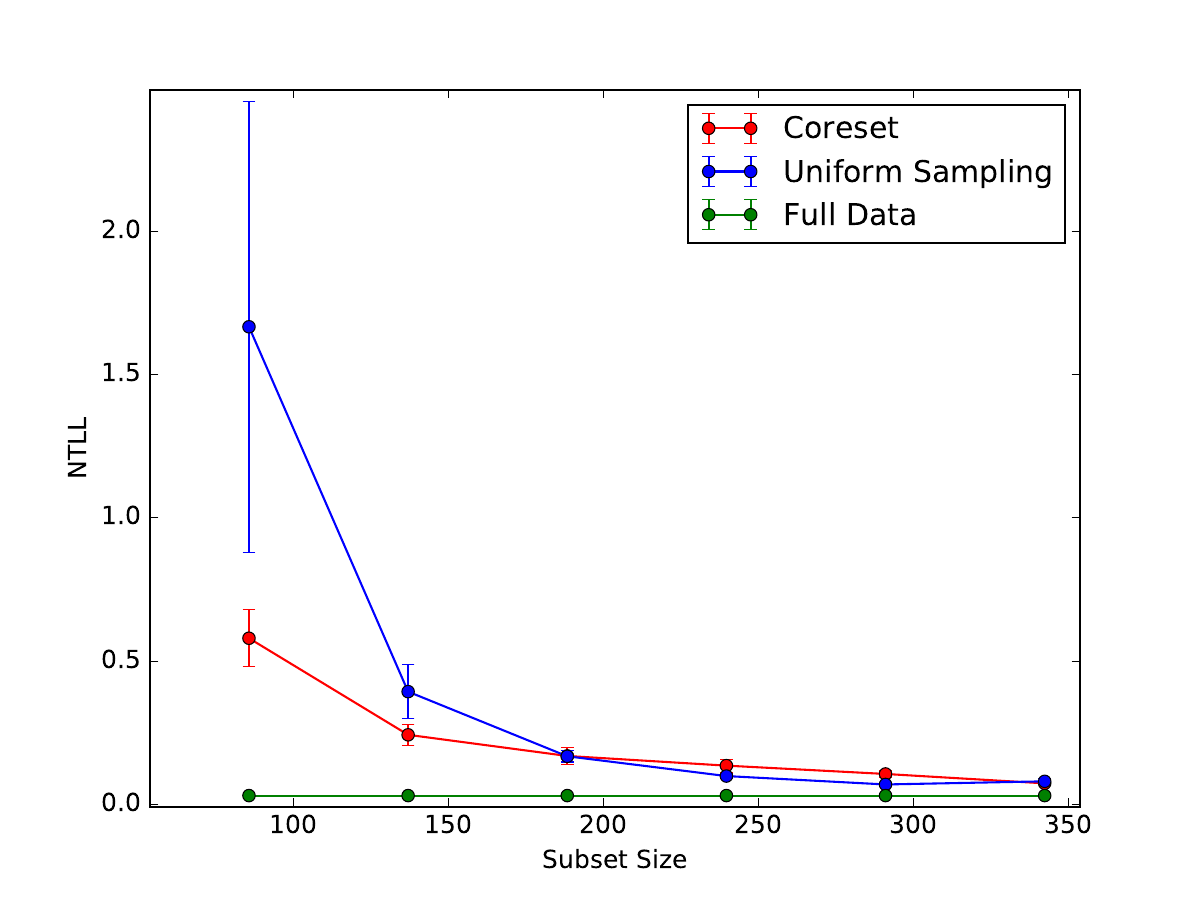}
		\caption{Wine dataset, $k=500,R=4$}
		\label{fig:e}
	\end{subfigure}
	\caption{Experimental results. \textbf{Fig.~\ref{fig:a}-\ref{fig:c}: }The error of maximizing sum of sigmoids using coreset and uniform sampling.
	\textbf{Fig.~\ref{fig:d}-\ref{fig:e}: }Negative test log-likelihood.
	Lower is better in all figures.\label{fig:Comparison-between-uniform}}
\end{figure}


\textbf{Discussion.}
As seen in Fig.~\ref{fig:Comparison-between-uniform}, our coreset outperforms the random sampling scheme as of accuracy, and is more stable (which can be seen in the standard deviation). For small sample sizes, the coreset provides a very small approximation error in practice, unlike the pessimistic theory which suggests bigger error. In this case, the coreset produces errors much smaller than the uniform sampling scheme.
As the sample sizes grow, both the coreset and the random sample simply contain a big portion of the original full data, and hence their output errors decrease and also becomes more similar, as predicted.
Furthermore, the coreset construction takes a neglectable amount of time from the total running time of computing a coreset and running the BFGS algorithm on the coreset. This is since the construction time is near linear. Hence, the computational time is not given in the graphs, as both sampling schemes required roughly the same total running time. 
Moreover, While in theory our results hold only for sufficiently large values of $k$, in practice we tested multiple $k$ values and witnessed a neglectable effect on the results. This is common in coresets paper where the worst-case theoretical bounds are too pessimistic and ignore structure in the data.

\vspace{-0.1cm}
\section{Conclusion}
\vspace{-0.1cm}
We provided a new coreset construction algorithm which computes a coreset for sums of sigmoid functions, which is common in deep learning and NP-hard to minimize, and logistic regression, where a coreset in~\cite{huggins2016coresetsLogistic} were suggested but with no support for regularization term, and no provable worst case bounds on the size of the coreset. Our construction algorithm is easily applicable to other functions as well. 
The coreset is of size near-logarithmic in the input size and can be computed in near-linear time.

Experimental results demonstrate that our coreset outperforms a standard sampling method, both in accuracy and stability. The experiments prove that empirically, our coreset is very effective; A coreset of size less than $1\%$ of the input suffices to produce a small error of $\varepsilon = 0.001$. 
Future work includes generalizing for additional widely common functions, and hopefully relaxing the assumptions required on the handled functions.

\bibliography{sigmoid}
\bibliographystyle{plain}

\clearpage

\appendix


\section{No Coreset for General Monotonic Functions} \label{sec:noCoreset}
In this section, we provide the full proofs behind the impossibility bound claims presented in Section~\ref{sec:lower_bounds}.

The following lemma proves that if the sensitivity of every input point is $1$ in a given query space, then there is no non-trivial coreset for the query space.
\begin{lemma}[Lower bound via Total sensitivity]
	\label{lemma:sen_is_nec}Let $(P,w,X,c)$ be a query space, and $\varepsilon\in\left(0,1\right)$. If every $\b{p}\in P$ has sensitivity $s_{P,w,X,c}\left(\b{p}\right)= 1$,
	then for every $\eps$-coreset $(Q,u)$ we have $Q=P$.
\end{lemma}
\ifproofs
\begin{proof}
	Let $(Q,u)$ be a weighted set, where $Q\subset P$.  It suffices to prove that $(Q,u)$ is not an $\eps$-coreset for $P$.
	Denote
\[
	u_{\max}  \in\arg\max_{\b{p}\in Q}u\left(\b{p}\right),
	\text{ and }
	w_{\min}\in\arg\min_{\b{p}\in P}w\left(\b{p}\right).
	\]
	Let $p\in P\setminus Q$. By the assumption $s_{P,w,X,c}\left(\b{p}\right)\geq 1$, there is $\b{x}_{\b{p}}\in X$ such
	that	
	\[ \frac{w\left(\b{p}\right)c\left(\b{p},\b{x}_{\b{p}}\right)}{C\left(P,w,\b{x}_{\b{p}}\right)}
= 1>\frac{u_{\max}}{u_{\max}}-\frac{w_{\min}\left(1-\varepsilon\right)}{u_{\max}}.
	\]
	Multiplication by $C(P,w,\b{x}_p)$ yields
	\begin{align}
	\begin{split}
	&w\left(\b{p}\right)c\left(\b{p},\b{x}_{\b{p}}\right)  > \\  & \frac{u_{\max}-w_{\min}\left(1-\varepsilon\right)}{u_{\max}}\cdot C\left(P,w,\b{x}_{\b{p}}\right).\label{eq:sens_bound}
	\end{split}
	\end{align}

We have that
\begin{align}
	\nonumber&C\left(Q,u,\b{x_{p}}\right)  = \sum_{\b{q}\in Q}u\left(\b{q}\right)c\left(\b{q},\b{x}_{\b{p}}\right)  \\
	\nonumber&=\sum_{\b{q}\in Q} \frac{u\left(\b{q}\right)}{w\left(\b{q}\right)}w\left(\b{q}\right)c\left(\b{p},\b{x}_{\b{p}}\right)
\le	\frac{u_{\max}}{w_{\min}}\sum_{\b{q}\in Q}w\left(\b{q}\right)c\left(\b{q},\b{x}_{\b{p}}\right)  \\
&	\label{byassum}\le\frac{u_{\max}}{w_{\min}}\sum_{\b{p'}\in P\setminus\left\{ \b{p}\right\} }w\left(\b{p}\right)c\left(\b{p'},\b{x_{p}}\right)  \\\nonumber&=\frac{u_{\max}}{w_{\min}}\left(C\left(P,w,\b{x}_{\b{p}}\right)-w\left(\b{p}\right)c\left(\b{p},\b{x_{p}}\right)\right) \\\label{fif}&<\frac{u_{\max}}{w_{\min}}C\left(P,w,\b{x}_{\b{p}}\right)\left(
1-\frac{u_{\max}-w_{\min}\left(1-\varepsilon\right)}{u_{\max}}\right)
\\\nonumber&=\left(1-\varepsilon\right)C\left(P,w,\b{x}_{\b{p}}\right),
\end{align}
	where~\eqref{byassum} is by the assumption $\b{p}\in P\setminus Q$, and~\eqref{fif} is by~\eqref{eq:sens_bound}.
Hence $Q$ cannot be used to approximate $C(P,w,\b{x_p})$ and thus is not an $\eps$-coreset for $P$.
\end{proof}
\fi

To prove there are query spaces $(P,w,X,c)$ which admit no non-trivial coreset, we are left to formally prove there exists a set of points for which the sensitivity of every point is $1$. Together with the lemma above, this will complete the proof. 

Similarly to the idea behind the counter example in Section~\ref{sec:lower_bounds}, the idea behind finding a set for which every point has sensitivity $1$ is to find a set of points in which every point is linearly separable from the rest of the set.
Such a set was shown to exist in \cite{huggins2016coresetsLogistic}.
\begin{lemma}[\cite{huggins2016coresetsLogistic}]
	\label{lemma:existence of seperation}
	There is a finite set of points $P\subseteq\REAL^d $
	such that for every $\b{p}\in P$ and $R>0$ there
	is $\b{y_{p}}\in\REAL^d$ of length $\norm{\b{y_p}}\leq R$ such that
$\b{y_{p}}\cdot\b{p}=-R$, and for every $\b{q}\in P\setminus \br{\b{p}}$ we have
$\b{y_{p}}\cdot\b{q}\ge R$.
\end{lemma}

The following theorem stems from the combination of the above claims. Consider the query space $(P,w,X,c)$, where $P$ is the set of points from the lemma above, and $c(\b{x},\b{p})=f(\b{x}\cdot\b{p})$ for every $\b{x},\b{p}\in\REAL^{d}$, where $f:\REAL\rightarrow\left(0,\infty\right)$ is a non-decreasing monotonic function. 
The theorem proves that, with respect to the query space $(P,w,X,c)$, the sensitivity of every point in $P$ is $1$.
We generalize a result from \cite{huggins2016coresetsLogistic} by considering weighted data and by letting the cost be any function upholding the conditions of Theorem~\ref{theorem:No coreset_appendix}.
\begin{theorem} [Theorem~\ref{theorem:No coreset}]
	\label{theorem:No coreset_appendix}Let $f:\REAL\rightarrow\left(0,\infty\right)$
	be a non-decreasing monotonic function that satisfies $\lim_{x\rightarrow\infty}\frac{f\left(-x\right)}{f\left(x\right)}=0$,
	and let $c\left(\b{x},\b{p}\right)=f\left(\b{x}\cdot\b{p}\right)$
	for every $\b{x},\b{p}\in\REAL^{d}$. Let $\varepsilon\in\left(0,1\right)$, $n\geq 1$ be an integer, and $w:\REAL^d\rightarrow\left(0,\infty\right)$. There is a set $P\subset\REAL^{d}$ of $|P|=n$ points such that if $(Q,u)$ is an $\eps$-coreset of $\left(P,w,\REAL^{d},c\right)$ then $Q=P$.
\end{theorem}
\ifproofs
\begin{proof}
	Let $P\subseteq\REAL^d$ be the set that is defined in Lemma~\ref{lemma:existence of seperation}, and let $\b{p}\in P$, and $R>0$.
By Lemma~\ref{lemma:existence of seperation}, there is $\b{y_p}\in \REAL^d$ such that $\b{y_p}\cdot \b{p}=-R$, and for every $\b{q}\in P\setminus \br{\b{p}}$ we have $\b{-y_{p}}\cdot\b{q} \le-R$.
	By this pair of properties,
	\[
f\left(-\b{y_{p}}\cdot\b{p}\right)=f\left(R\right)
\text{ and }
	f\left(\b{-y_{p}}\cdot\b{q}\right)\le f\left(-R\right),
	\]
where in the last inequality we use the assumption that $f$ is non-decreasing.
	By letting $\b{x_{p}=-}\b{y_{p}}$, we have
	\[
	\frac{w(q)f\left(\b{x_{p}}\cdot\b{q}\right)}{w(p)f\left(\b{x_{p}}\cdot\b{p}\right)}
	=\frac{w(q)f\left(-\b{y_{p}}\cdot\b{q}\right)}{w(p)f\left(\b{-y_{p}}\cdot\b{p}\right)}
\le\frac{w(q)f\left(-R\right)}{w(p)f\left(R\right)}.
	\]
	Therefore, by letting $w_{\max}\in\arg\max_{\b{p}\in P}w\left(\b{p}\right)$,
	\begin{align*}
	&s_{P,w,\REAL^d,c}\left(\b{p}\right)  \geq \frac{w\left(\b{p}\right)
f\left(\b{x_p}u\cdot\b{p}\right)}{\sum_{\b{q}\in P}w\left(\b{q}\right)f\left(\b{x_p}\cdot\b{q}\right)}  \\ & =\frac{w\left(\b{p}\right)f\left(\b{x_{p}}\cdot \b{p}\right)}{w\left(\b{p}\right)f\left(\b{p}\cdot\b{x_{p}}\right)
+\sum_{\b{q}\in P\setminus\br{\b{p}}}w\left(\b{q}\right)f\left(\b{x_{p}}\cdot\b{q}\right)} \\
 &=\frac{1}{1+\sum_{\b{q}\in P\setminus\br{\b{p}}}\frac{w\left(\b{q}\right)f\left(\b{x_{p}}\cdot\b{q}\right)}{w\left(\b{p}\right)f\left(\b{x_{p}}\cdot \b{p}\right)}} 
  \ge \frac{1}{1+\sum_{q\in P\setminus\br{\b{p}}}\frac{w\left(\b{q}\right)f\left(-R\right)}{w\left(\b{p}\right)f\left(R\right)}} \\ &\ge \frac{1}{1+\left(n-1\right)\frac{w_{\max}f\left(-R\right)}{w\left(\b{p}\right)f\left(R\right)}}.
	\end{align*}
	We also have
	\[
\lim_{R\rightarrow\infty}\frac{	w_{\max} f\left(-R\right)}{w\left(\b{p}\right)f\left(R\right)}
=	\frac{w_{\max}}{w(\b{p})}\lim_{R\rightarrow\infty}\frac{f\left(-R\right)}{f\left(R\right)}=0,
\]
where the last derivation holds by the assumption on $f$.
Thus we obtain
\[
s_{P,w,\REAL^{d},c}\left(\b{p}\right)
=\sup_{R>0}\frac{1}{1+\left(n-1\right)\frac{	w_{\max} f\left(-R\right)}{w\left(\b{p}\right)f\left(R\right)}}=1.
\]
Theorem~\ref{theorem:No coreset_appendix} then follows from the last equality and Lemma~\ref{lemma:sen_is_nec}.
\end{proof}
\fi

\section{$\mathcal{L}_\infty$-Coresets} \label{sec:LInfCoresets}

\begin{lemma} [Lemma~\ref{lemma:exsitance of inf bounded}] \label{lemma:exsitance of inf bounded_appendix}
Let $P\subset\REAL^{d}$ be a finite set, $M, k>0$ be constants, $f:\REAL\rightarrow(0,M]$
be non-decreasing function and $g:[0,\infty)\rightarrow[0,\infty)$ be a function. For every $\b{x}\in\REAL^{d}$ and $\b{p}\in P$ define $c_{k}\left(\b{p,x}\right)=f\left(\b{p}\cdot\b{x}\right)+\frac{g\left(\left\Vert \b{x}\right\Vert \right)}{k}$.
Put $\b{p}\in P$. Suppose there is $b_{\b{p}}>0$ such that for every $z>0$
\begin{equation}
f\left(\left\Vert \b{p}\right\Vert z\right)+\frac{g\left(z\right)}{k}\le b_{\b{p}}\left(f\left(-\left\Vert \b{p}\right\Vert z\right)+\frac{g\left(z\right)}{k}\right).\label{eq:assumption 2-1}
\end{equation}
Then $\br{\b{p}}$ is an $\eps-\mathcal{L}_\infty$ coreset with $\eps = \frac{M}{f\left(0\right)}\left(b_{\b{p}}+1\right)-1$, i.e., for every $\b{x}\in\REAL^{d}$
\[
\max_{\b{p'}\in P}c_{k}\left(\b{p'},\b{x}\right) \le \frac{M}{f\left(0\right)}\left(b_{\b{p}}+1\right)c_{k}\left(\b{p},\b{x}\right).
\]
\end{lemma}
\ifproofs
\begin{proof}
	Let $\b{x}\in\REAL^{d}$ and $q\in P$ 
	such that $\b{x}\cdot\b{q}>0$. We have, by the monotonic
	properties of $f$,
	\begin{equation}
	 f\left(0\right) \leq f\left(\b{x}\cdot\b{q}\right).\label{eq:lower_bound_q}
	\end{equation}
	Hence,
	\begin{equation}
	\max_{\b{p}'\in P}f\left(\b{x}\cdot\b{p}'\right)\le M=\frac{M}{f\left(0\right)}f\left(0\right)\le\frac{M}{f\left(0\right)}f\left(\b{x}\cdot\b{q}\right),\label{eq:case1}
	\end{equation}
	where the first inequality is since $f$ is bounded by $M$, and the
	last inequality is by~\eqref{eq:lower_bound_q}. By adding $\frac{g\left(\left\Vert \b{x}\right\Vert \right)}{k}$
	to both sides of~\eqref{eq:case1} and since $1\leq \frac{M}{f\left(0\right)}$
	we obtain,
	\begin{align}
	\begin{split}
&\max_{p'\in P}c_k(p',x)=	\max_{\b{p}'\in P}f\left(\b{x}\cdot\b{p}'\right)+\frac{g\left(\left\Vert \b{x}\right\Vert \right)}{k} \\
&\le
	\frac{M}{f\left(0\right)}f\left(\b{x}\cdot\b{q}\right)
+\frac{g\left(\left\Vert \b{x}\right\Vert \right)}{k}\\
&\leq	\frac{M}{f\left(0\right)}\left(f\left(\b{x}\cdot\b{q}\right)+\frac{g\left(\left\Vert \b{x}\right\Vert \right)}{k}\right).\label{eq:upper_bound_P'-1}
	\end{split}
	\end{align}
	The rest of the proof follows by case analysis on the sign of $\b{x}\cdot\b{p}$,
	i.e. $\left(i\right)\,\b{x}\cdot\b{p}\ge0$ and
	$\left(ii\right)\,\b{x}\cdot\b{p}<0$.
	
	\textbf{Case (\romannum{1})}: $\b{x}\cdot\b{p}\ge0$. 	Substituting $q=p$ in ~\eqref{eq:upper_bound_P'-1} yields
	\begin{align}
	\begin{split}
&\max_{p'\in P}c_k(p',x)\le 	 \frac{M}{f\left(0\right)}\left(f\left(\b{x}\cdot\b{p}\right)+\frac{g\left(\left\Vert \b{x}\right\Vert \right)}{k}\right)\label{eq:case1-1}\\
&=	 \frac{M}{f\left(0\right)}c_k(p,x)\leq \frac{M}{f\left(0\right)}(b_p+1)c_k(p,x),
 	\end{split}
	\end{align}
where the last inequality follows by the assumption $b_p>0$.
	\textbf{Case (\romannum{2})}: $\b{x}\cdot\b{p}<0$. In this case $\b{x}\cdot\left(-\b{p}\right)>0$. Substituting $q=-p$ in~\eqref{eq:upper_bound_P'-1} yields
	\begin{align}
&\max_{p'\in P}c_k(p',x) \le
	\frac{M}{f\left(0\right)}\left(f\left(\b{x}\cdot\left(-\b{p}\right)\right)+\frac{g\left(\left\Vert \b{x}\right\Vert \right)}{k}\right)\label{eq:6}   \\
& \le	\frac{M}{f\left(0\right)}\left(f\left(\left\Vert \b{x}\right\Vert \left\Vert \b{p}\right\Vert \right)+\frac{g\left(\left\Vert \b{x}\right\Vert \right)}{k}\right)\label{eq:7} \\
	&\leq \frac{M}{f\left(0\right)}b_{\b{p}}\left(f\left(-\left\Vert \b{x}\right\Vert \left\Vert \b{p}\right\Vert \right)+\frac{g\left(\left\Vert \b{x}\right\Vert \right)}{k}\right)\label{eq:8} \\
&\le\frac{M}{f\left(0\right)}b_{\b{p}}\left(f\left(\b{x}\cdot\b{p}\right)+\frac{g\left(\left\Vert \b{x}\right\Vert \right)}{k}\right),\label{eq:9}\\
&=\frac{M}{f(0)}b_pc_k(p,x)\leq \frac{M}{f(0)}(b_p+1)c_k(p,x),
	\end{align}
	where~\eqref{eq:7} and~\eqref{eq:9}
	are by the Cauchy-Schwartz inequality and the monotonicity of $f$, and~\eqref{eq:8} follows by substituting $z=\norm{x}$ in the main assumption of the lemma.
\end{proof}
\fi

\subsection{From $\varepsilon-\mathcal{L}_\infty$ coresets to $\varepsilon$-coresets}

In what follows is the full proof for Lemma~\ref{lem:reduction}. We prove that the algorithm described in Section~\ref{sec:fromLInfToCoresets}, which constructs a series of $\mathcal{L}_\infty$ coresets, can indeed give an upper bound on the sensitivity of every input element as well as a near logarithmic upper bound on the total sensitivity.
\begin{lemma} [Lemma~\ref{lem:reduction}] \label{lem:reduction_appendix}
Let $c:\REAL^d \times \REAL^d \to (0, \infty)$. Suppose that for some $\eps \in(0,1)$ there is a non-decreasing function $\Delta_{\eps}(n)$ so that for any $P' \subseteq \REAL^d$ of size $n$ there is an $\eps-\mathcal{L}_\infty$ coreset of size at most $\Delta_{\eps}(n)$ for $(P', \mathbf{1}, \REAL^d, c)$. Then, for any $P \subseteq \REAL^d$ of size $n$ we can compute an upper bound $s(p)$ on the sensitivity $s_{P,\mathbf{1},\REAL^d,c}(p)$ for each $p\in P$, so that $\sum_{p\in P}s_{P,\mathbf{1},\REAL^d,c}(p) \leq (1+\eps)\Delta_{\eps}(n)\ln{n}$. 
\end{lemma}
\ifproofs
\begin{proof}
The proof is constructive. We build a sequence of subsets $P_1 \supseteq P_2 \supseteq \cdots \supseteq P_m$, where $P = P_1$, $m\leq n$, and $|P_m| \leq \Delta_{\eps}(n)$. We construct the sequence as follows. If $|P_i| \leq \Delta_{\eps}(n)$ the sequence stops. Otherwise, we compute an $\mathcal{L}_\infty$ $\eps$-coreset $C_i$ for $(P_i, \mathbf{1}, \REAL^d, c)$ of size $|C_i| \leq \Delta_{\eps}(n)$. We now define $P_{i+1} = P_i \setminus C_i$.

Put $i \in [m]$. We now upper bound the sensitivity $s_{P,\mathbf{1},\REAL^d,c}(q)$ for every $q\in C_i$ by $\frac{1+\eps}{i}$ and upper bound the total sensitivity $\sum_{p\in P}s(p) \leq (1+\eps)\Delta_{\eps}(n)\ln{n}$.

Put $x\in \REAL$ and $q' \in C_i$, and consider $1 \leq j \leq i$. Let $q_j \in C_j$ be the points in the $\eps-\mathcal{L}_\infty$ coreset $C_j$ such that $c(q_j, x) = \max_{q \in C_j} c(q,x)$. We now have that
\begin{equation} \label{eq:qj}
c(q',x) \leq \max_{p \in P_j} c(p,x) \leq (1+\eps) \max_{q \in C_j} c(q,x) = (1+\eps) c(q_j, x)
\end{equation}
where the first derivations holds since, by construction, $q' \in P_j$. The second derivation is by the definition of an $\eps-\mathcal{L}_\infty$ coreset.
We thus obtain that
\begin{equation} \label{eq:sens_qj}
\frac{c(q',x)}{\sum_{p \in P} c(p,x)} \leq \frac{c(q',x)}{\sum_{\ell=1}^i c(q_\ell,x)} \leq \frac{1+\eps}{i},
\end{equation}
where the second derivation holds since $\br{q_j \mid 1\leq j \leq i} \subseteq P$, and the last derivation is by~\eqref{eq:qj}.
Since~\eqref{eq:sens_qj} holds for any $x \in \REAL^d$, we obtain that the sensitivity of $q'$ is upper bounded by
\[
s_{P,\mathbf{1},\REAL^d,c}(q') = \sup_{\b{x}\in \REAL^d}\frac{c(q',x)}{\sum_{p \in P} c(p,x)} \leq \frac{1+\eps}{i}.
\]
Hence, for every $q' \in C_i$ we have that $s_{P,\mathbf{1},\REAL^d,c}(q') \leq \frac{1+\eps}{i}$. Now, the total sensitivity can be bounded by
\[
\sum_{p\in P} s_{P,\mathbf{1},\REAL^d,c}(p) = \sum_{i=1}^m \frac{(1+\eps)|C_i|}{i} \leq \Delta_{\eps}(n) \sum_{i=1}^m \frac{1+\eps}{i} \leq (1+\eps)\Delta_{\eps}(n) \ln{n}.
\]
\end{proof}\fi

\subsection{Coreset sufficient condition} \label{sec:coresetSuffCond}

In what follows we give the full proof for Theorem~\ref{theorem:coreset}. The proof is constructive in the sense that it gives an upper bound for the sensitivity of every input point and upper bounds the total sensitivity by a term which is near logarithmic in the input size.
\begin{theorem} [Theorem~\ref{theorem:coreset}] \label{theorem:coreset_appendix}
Let $M, k>0$ be constants, $P \subseteq \REAL^d$ be a set of points, $f:\REAL\to (0,M]$ be a monotonic non-decreasing function, and $c_{k}\left(\b{p',x}\right)=f\left(\b{p'}\cdot\b{x}\right)+\frac{g(\norm{x})}{k}$ for every $\b{x}\in\REAL^{d}$ and $\b{p'}\in P$. 
Suppose there is $b:P \to (0,\infty)$ such that for every $\b{p}\in P$ and every $z>0$
\begin{equation}
f\left(\norm{\b{p}} z\right)+\frac{g\left(z\right)}{k}\le b({\b{p}})\left(f\left(-\norm{\b{p}} z\right)+\frac{g\left(z\right)}{k}\right).
\end{equation}
Let $b_{\max} \in \argmax_{\b{p} \in P} b({\b{p})}$, $t=(1+\frac{M}{f(0)} b_{\max})\ln n$, and $\eps,\delta\in(0,1)$. Lastly, let $d_{VC}$ be the VC-dimension of $(P,\mathbf{1},\REAL^d,c_k)$.
Then, there is a weighted set $(Q,u)$, where $Q\subseteq P$ and
\[
|Q| \in O\left(\frac{t}{\varepsilon^{2}}\left(d_{VC}\log t+\log\frac{1}{\delta}\right)\right),
\]
such that with probability at least $1-\delta$, $(Q,u)$ is an $\eps$-coreset for the query space $(P,\mathbf{1},\REAL^d,c_k)$.
\end{theorem}
\ifproofs
\begin{proof}
For $\b{p} \in P$ we have that	
\begin{align}
\max_{\b{p'}\in \label{eq:10aa} P}c_{k}\left(\b{p'},\b{x}\right)\le\frac{M}{f\left(0\right)}\left(b_{\b{p}}+1\right)c_{k}(\b{p},\b{x}) \le  \\ 
\frac{M}{f\left(0\right)}\left(b_{\max}+1\right)c_{k}(\b{p},\b{x}) \label{eq:10ba}.
\end{align}
Where~\eqref{eq:10aa} is by substituting in Lemma~\ref{lemma:exsitance of inf bounded} and~\eqref{eq:10ba} holds since for every $\b{p}\in P$, $b({\b{p}}) \le b_{\max}$. 
Let $\eps(\b{p}) := \left[\left(\frac{M}{f\left(0\right)}\left(b(\b{p})+1\right)\right) - 1\right]$ for every $\br{p}\in P$ and let $\eps = \left(\frac{M}{f\left(0\right)}\left(b_{\max}+1\right)\right) - 1$. 
Thus, for every $\b{p} \in P$, we have that $\{\b{p}\}$ is an $\eps(\b{p})$-$\mathcal{L}_{\infty}$ coreset which is a $\eps$-$\mathcal{L}_{\infty}$ coreset.

In Lemma~\ref{lem:reduction}, a sequence of distinct $\mathcal{L}_\infty$ coresets that cover the entire set P $C_1 \cup\cdots\cup C_m \subseteq P$ are constructed. 
For every $\b{p} \in P$, let $i(\b{p})$ be the index of the $\mathcal{L}_\infty$ coreset $C_{i(\b{p})}$ such that $\b{p} \in C_{i(\b{p})}$.
Plugging $\eps$ and $\Delta_\eps(n) = 1$ in Lemma~\ref{lem:reduction} and its proof yields that we can upper bound the sensitivity of every $\b{p} \in P$ by 
\[
s_{P,\mathbf{1},\REAL^d,c_k}(\b{p}) \leq \frac{1+\eps(\b{p})}{i(\b{p})} \leq \frac{1+\eps}{i(\b{p})} = \frac{\left(\frac{M}{f\left(0\right)}\left(b_{\max}+1\right)\right)}{i(\b{p})},
\]
where $i(\b{p})$ is the index of $\b{p}$ when sorting the points in $P$ by their norm.
Furthermore, the total sensitivity is bounded by
\[
t(P,\b{1},\REAL^d,c_k) = \sum_{\b{p} \in P} \frac{\left(\frac{M}{f\left(0\right)}\left(b_{\max}+1\right)\right)}{i(\b{p})} \in O\left((1+\frac{M}{f(0)} b_{\max})\ln n\right).
\]
Observe that sensitivity of $\b{p} \in P$ depends on $\eps(\b{p}) = \left[\left(\frac{M}{f\left(0\right)}\left(b(\b{p})+1\right)\right) - 1\right]$ divided by the index $i(\b{p})$. Hence, empirically, to obtain smaller total sensitivity, we would prefer to reorder $P$ such that points $\b{p}$ with larger value of $b(\b{p})$ are divided by larger values $i(\b{p})$. Therefore, we can simply sort the points of $P$ according to the values of the function $b$, from smallest to largest. Thus, points $\b{p}$ with larger value of $b(\b{p})$ are given larger index $i(\b{p})$.

Theorem~\ref{theorem:coreset} now immediately follows from Theorem~\ref{theorem:sens_is_coreset}.
\end{proof}
\fi

\section{Main Proofs}
In this section, we first prove a series of technical claims. We then utilize those claims to prove the main results of this work.

\subsection{Technical Claims}
\begin{lemma} \label{lemma:sigmoid_intersection}
	Let $f:\REAL \to (0, \infty)$ be a monotonic increasing function such that $f(0) > 0$. Let $c,k>0$. There is exactly one number $x_{kc}>0$
	that simultaneously satisfies the following claims.
	
	\begin{enumerate}[label=\textbf{(\roman*)}]
		\item \label{sig.1} $f\left(-\sqrt{ck}x_{kc}\right)=x_{kc}^{2}.$
		\item \label{sig.2} For every $x>0$, if $f\left(-\sqrt{ck}x\right)>x^{2}$ then $x<x_{kc}$. 
		\item \label{sig.3} For every $x>0$, if $f\left(-\sqrt{ck}x\right)<x^{2}$ then $x>x_{kc}$.
		\item \label{sig.4} There is $k_{0}>0$ such that for every $k' \ge k_{0}$ 
		\[
		\frac{1}{x_{kc}}\le \sqrt{ck'}.
		\]
	\end{enumerate}
\end{lemma}
\begin{proof}
	Let $g(x)=x^2$. Define 
	\begin{equation}
	h_{kc}(x)=f(-\sqrt{ck}x)-g(x).\label{eq:24}
	\end{equation}
	
	\textbf{(\romannum{1})}: It holds that
	\begin{equation}
	h_{kc}(0)=f(0)\label{eq:25-1}
	\end{equation}
	and 
	\begin{equation}
	h_{kc}\left(\sqrt{f(0) + 1}\right)<0,\label{eq:25}
	\end{equation}
	where~\eqref{eq:25} holds since $f\left(-\sqrt{ck}x\right) \le f(0)$ for every $x>0$, and $g\left(\sqrt{f(0) + 1}\right)=f(0) + 1$. From~\eqref{eq:25-1}
	and~\eqref{eq:25} we have that $0\in\left[h_{kc}\left(\sqrt{f\left(0\right) + 1}\right),h_{kc}(0)\right].$
	Using the Intermediate Value Theorem (Theorem~\ref{theorem:Intermediate-Value-Theorem})
	we have that there is $x_{1}\in\left(0,\sqrt{f(0) + 1}\right)$ such that 
	\begin{equation}
	h_{kc}(x_{1})=0.\label{eq:35}
	\end{equation}
	We prove that $x_{1}$ is unique. By contradiction. Assume that there
	is $x_{2}\neq x_{1}$ such that 
	\begin{equation}
	h_{kc}(x_1)=h_{kc}(x_2)=0.\label{eq:27-3}
	\end{equation}
	Wlog assume that $x_{1}<x_{2}.$ By The Mean Value Theorem (Theorem~\ref{theorem:Mean-Value-Theorem}),
	there is $r\in\left(x_{1},x_{2}\right)$ such that 
	\begin{align}
	h'_{kc}(r) & =\frac{h_{kc}(x_{2})-h_{kc}(x_{1})}{x_{2}-x_{1}}\\
	& =0,\label{eq:29-1}
	\end{align}
	where~\eqref{eq:29-1} is by~\eqref{eq:27-3}.
	The derivative of $h_{kc}$ is
	\begin{align}
	h'_{kc}(x)= & \left(f\left(-\sqrt{ck}x\right)-g(x)\right)'\label{eq:28}\\
	= & -\sqrt{ck}f'\left(-\sqrt{ck}x\right)-g'(x) < 0,\label{eq:27-1}
	\end{align}
	where~\eqref{eq:28} is by~\eqref{eq:24} and~\eqref{eq:27-1} is since $f$ is monotonic increasing and thus $f'(x) > 0$ for every $x\in\REAL$ and $x,k,c>0$.~\eqref{eq:27-1} is a contradiction to~\eqref{eq:29-1}. Thus the Assumption~\eqref{eq:27-3} is false and $x_{1}$ is unique. 
	
	By~\eqref{eq:24} and~\eqref{eq:35}
	\begin{equation}
	f\left(-\sqrt{ck}x_{1}\right)=g(x_{1}).\label{eq:36-2}
	\end{equation}
	By letting $x_{kc}=x_1$ and recalling that $g(x)=x^2$
	we obtain
	\[
	f\left(-\sqrt{ck}x_{kc}\right)=x_{kc}^2.
	\]
	
	\textbf{(\romannum{2})}: Let $x>0$ such that $f\left(-\sqrt{ck}x\right)>x^{2}$.
	Plugging this and the definition $g(x)=x^{2}$ in~\eqref{eq:24}
	yields
	\begin{equation}
	h_{kc}(x)>0.\label{eq:37-2}
	\end{equation}
	We already proved that $h'_{kc}(x)<0$ always. By the Inverse
	of Strictly Monotone Function Theorem (Theorem~\ref{theorem:Inverse-of-Strictly})
	we have that the inverse $h_{kc}^{-1}$ of $h_{kc}$ is a strictly
	monotone decreasing function. Applying $h_{kc}^{-1}$ on both sides
	of~\eqref{eq:37-2} gives
	\[
	x<x_{kc}.
	\]
	
	\textbf{(\romannum{3})}: Let $x>0$ such that $f\left(-\sqrt{ck}x\right)<x^{2}$.
	By this and by the definition of $g$ and~\eqref{eq:24} we have
	\begin{equation}
	h_{kc}(x)<0.\label{eq:37-2-1}
	\end{equation}
	We already proved that $h'_{kc}(x)<0$ always. By the Inverse of Strictly Monotone Function Theorem (Theorem~\ref{theorem:Inverse-of-Strictly})
	we have that $h_{kc}$ has a strictly monotone decreasing inverse
	function $h_{kc}^{-1}$. Applying $h_{kc}^{-1}$ on both sides of~\eqref{eq:37-2-1} gives
	\[
	x>x_{kc}.
	\]
	
	\textbf{(\romannum{4})}: We need to prove that there is $k_{0}$ such that for every $k' > k_{0}$ we have
	\begin{equation}
	x_{kc} \ge \frac{1}{\sqrt{ck'}}
	\end{equation}
	
	By contradiction, assume that for every $k' > 0$,
	\begin{equation}
	x_{kc} < \frac{1}{\sqrt{ck'}}. \label{eq:contradiction_assumption}
	\end{equation}
	Since $f$ is increasing and by~\eqref{eq:contradiction_assumption} $-c\sqrt{k}x_{kc} > -1$ we have that
	\begin{equation}
	f\left(-\sqrt{ck}x_{kc}\right) > f(-1). \label{eq:16-1}
	\end{equation}
	where~\eqref{eq:16-1} holds since $f$ is increasing and by~\eqref{eq:contradiction_assumption} $-\sqrt{ck}x_{kc} > -1$.
	Since $\lim_{k\to\infty} \frac{1}{ck} = 0$, there is $k_{0} >0$ such that for every $k > k_{0}$
	\begin{equation}
	\begin{split}
	f\left(-1\right) & > \frac{1}{ck}\\
	& > x_{kc}^{2}, \label{eq:a}
	\end{split}
	\end{equation}
	where~\eqref{eq:a} is by~\eqref{eq:contradiction_assumption}. Plugging~\eqref{eq:a} in~\eqref{eq:16-1} yields
	\begin{equation}
	f\left(-c\sqrt{k}x_{kc}\right) > x_{kc}^{2}.
	\end{equation}
	In contradictions to~\ref{sig.1}. Thus
	\begin{equation}
	x_{kc} \ge \frac{1}{\sqrt{ck}}
	\end{equation}
\end{proof}

\begin{lemma}\label{lemma:ratio_simple}
	Let $f$ be either the sigmoid or the logistic regression function and let $x_{1,1} > 0$ which is obtained by applying Lemma~\ref{lemma:sigmoid_intersection}\ref{sig.1} with $f$ and $k=c=1$. Then, For every $x\ge0$ 
	\[
	\frac{f\left(x\right)+x^{2}}{f\left(-x\right)+x^{2}} \le \max\left\{2, \frac{2}{x_{1,1}^{2}}\right\}.
	\]
\end{lemma}
\begin{proof}
	Let $x\ge0$. Substituting $k=c=1$ in Lemma~\ref{lemma:sigmoid_intersection}\ref{sig.1} yields that $f\left(-x_{1,1}\right)=x_{1,1}^{2}$. We show that $\frac{f\left(x\right)+x^{2}}{f\left(-x\right)+x^{2}} \le \max\left\{2, \frac{2}{x_{1,1}^{2}}\right\}$
	via the following case analysis. \textbf{(\romannum{1})} $f\left(x\right)\ge x^{2}$ and $f\left(-x\right)\ge x^{2}$, \textbf{(\romannum{2})} $f\left(x\right)\ge x^{2}$ and $f\left(-x\right)<x^{2}$, \textbf{(\romannum{3})} $f\left(x\right)<x^{2}$and $f\left(-x\right)\ge x^{2}$, and  \textbf{(\romannum{4})} $f\left(x\right)<x^{2}$and $f\left(-x\right)<x^{2}$.
	
	\textbf{Case (i)}: $f\left(x\right)\ge x^{2}$ and $f\left(-x\right)\ge x^{2}$. Since $f\left(-x\right)\ge x^{2}$, by substituting $k=c=1$ in Lemma~\ref{lemma:sigmoid_intersection}\ref{sig.2},
	we have that $x\le x_{1,1}$. Hence
	\begin{align}
	f\left(-x\right)+x^{2}\ge & f\left(-x\right)\label{eq:16}\\
	\ge & f\left(-x_{1,1}\right)\label{eq:17}\\
	= & x_{1}^{2},\label{eq:18}
	\end{align}
	where~\eqref{eq:16} is since $x^{2}>0$,~\eqref{eq:17} is since $f$
	is increasing and $x\le x_{1,1}$, and~\eqref{eq:18} is by definition
	of $x_{1,1}$. By adding $f\left(x\right)$ to both sides of the
	assumption $f\left(x\right)\ge x^{2}$ of Case (\romannum{1}) we obtain 
	\begin{equation}
	2f\left(x\right)\ge f\left(x\right)+x^{2}.\label{eq:51}
	\end{equation}
	By~\eqref{eq:51} and~\eqref{eq:18} we obtain
	\begin{equation}
	\frac{f\left(x\right)+x^{2}}{f\left(-x\right)+x^{2}}\le\frac{2f\left(x\right)}{x_{1,1}^{2}}\le\frac{2}{x_{1,1}^{2}} \le \max\left\{2, \frac{2}{x_{1,1}^{2}}\right\}.\label{eq:case_1}
	\end{equation}
	where the second inequality holds since $f(x) \leq 1$ due to $f$ being the sigmoid function. 
	
	\textbf{Case (\romannum{2})}: $f\left(x\right)\ge x^{2}$ and $f\left(-x\right)<x^{2}$. Since $f\left(-x\right)<x^{2}$, substituting $k=c=1$ in Lemma~\ref{lemma:sigmoid_intersection}\ref{sig.3}, there is $x_{1,1}$ such that
	\begin{align}
	f\left(-x\right)+x^{2}\ge & x^{2}\label{eq:19}\\
	> & x_{1,1}^{2},\label{eq:20}
	\end{align}
	where~\eqref{eq:19} is since $f$ is a positive function and~\eqref{eq:20} is since $x>x_{1,1}$
	. By adding $f\left(x\right)$ to both sides of the assumption $f\left(x\right)\ge x^{2}$ of Case (\romannum{2}) we have that 
	\begin{equation}
	f\left(x\right)+x^{2} \le 2f\left(x\right).\label{eq:54}
	\end{equation}
	By~\eqref{eq:54}) and~\eqref{eq:20} we obtain
	\begin{equation}
	\frac{f\left(x\right)+x^{2}}{f\left(-x\right)+x^{2}}\le\frac{2f\left(x\right)}{x_{1,1}^{2}}\le\frac{2}{x_{1,1}^{2}} \le \max\left\{2,\frac{2}{x_{1,1}^{2}}\right\}.\label{eq:case_2}
	\end{equation}
    where the second inequality holds since $f(x) \leq 1$ due to $f$ being either the sigmoid or the logistic regression function. 
	
	\textbf{Case (\romannum{3})}: $f\left(x\right)<x^{2}$and $f\left(-x\right)\ge x^{2}$. By adding $x^{2}$ to both sides of the assumption $f\left(x\right)<x^{2}$ of Case (\romannum{3}) we have that 
	\begin{equation}
	f\left(x\right)+x^{2}\le2x^{2}.\label{eq:55}
	\end{equation}
	Furthermore, since $f\left(-x\right)>0$ we have that
	\begin{equation}
	f\left(-x\right)+x^{2}\ge x^{2}.\label{eq:56}
	\end{equation}
	Combining~\eqref{eq:55} and~\eqref{eq:56} we obtain
	\begin{equation}
	\frac{f\left(x\right)+x^{2}}{f\left(-x\right)+x^{2}}\le\frac{2x^{2}}{x^{2}}\le2 \le \max\left\{2,\frac{2}{x_{1,1}^{2}}\right\}.\label{eq:case_3}
	\end{equation}
	
	\textbf{Case (\romannum{4})}: $f\left(x\right)<x^{2}$ and $f\left(-x\right)<x^{2}$. By adding $x^{2}$ to both sides of the assumption $f\left(x\right)<x^{2}$ of Case (\romannum{4}) we have that 
	\begin{equation}
	f\left(x\right)+x^{2}\le2x^{2}.\label{eq:57}
	\end{equation}
	Furthermore, since $f\left(-x\right)>0$ we have that
	\begin{equation}
	f\left(-x\right)+x^{2}\ge x^{2}.\label{eq:58}
	\end{equation}
	Combining~\eqref{eq:57} and~\eqref{eq:58} we obtain
	\begin{equation}
	\frac{f\left(x\right)+x^{2}}{f\left(-x\right)+x^{2}}\le\frac{2x^{2}}{x^{2}}\le2  \le \max\left\{2,\frac{2}{x_{1,1}^{2}}\right\}.\label{eq:case_4}
	\end{equation}
	
	Combining the results of the case analysis:~\eqref{eq:case_1},~\eqref{eq:case_2},~\eqref{eq:case_3},and~\eqref{eq:case_4} we have that 
	\begin{equation}
	\frac{f\left(x\right)+x^{2}}{f\left(-x\right)+x^{2}} \le \max\left\{2,\frac{2}{x_{1,1}^{2}}\right\}.\label{eq:bottom_line_simple}
	\end{equation}
\end{proof}

\begin{lemma}\label{lemma:L2RegBoundSig}
	Let $f$ be the sigmoid function, let $x_{1,1}$ be as in Lemma~\ref{lemma:ratio_simple}, and let $c>0$. Assume that there is $D > 1$ such that $\frac{f\left(cy\right)}{f\left(\frac{y}{\sqrt{k}}\right)} \leq D$ for every $y \ge 0$ and $k>0$. Then, there is $k_{0}>0$ such that for every $k\ge k_{0}$ and for every $x\ge0$, 
	\[
	\frac{f\left(\sqrt{c}x\right)+\frac{x^{2}}{k}}{f\left(-\sqrt{c}x\right)+\frac{x^{2}}{k}} \le 3D\max\left\{2,\frac{2}{x_{1,1}^{2}}\right\}ck.
	\]
\end{lemma}
\begin{proof}
	Let $x \ge 0$ and $k,c > 0$. We have that
	\begin{align}
	f\left(cx\right)+\frac{x^{2}}{k}\le & Df\left(\frac{x}{\sqrt{k}}\right)+\frac{x^{2}}{k}\label{eq:37}\\
	\le & D\max\left\{2,\frac{2}{x_{1,1}^{2}}\right\}\left(f\left(-\frac{x}{\sqrt{k}}\right)+\frac{x^{2}}{k}\right),\label{eq:reduction_to_neg}
	\end{align}
	where~\eqref{eq:37} holds since$\frac{f\left(cy\right)}{f\left(\frac{y}{\sqrt{k}}\right)} < D$ for every $y \ge 0$ and~\eqref{eq:reduction_to_neg} holds since $\frac{x^{2}}{k} \le D\frac{x^{2}}{k}$, and since, by Lemma~\ref{lemma:ratio_simple}, for every positive $z$ we have that 
	\[
	\frac{f\left(z\right)+z^{2}}{f\left(-z\right)+z^{2}} \le \max\left\{2,\frac{2}{x_{1,1}^{2}}\right\}.
	\]
	\\Dividing~\eqref{eq:reduction_to_neg} by $f\left(-\sqrt{c}x\right)+\frac{x^{2}}{k}$ yields 
	\begin{equation}
	\frac{f\left(cx\right)+\frac{x^{2}}{k}}{f\left(-\sqrt{c}x\right)+\frac{x^{2}}{k}}\le D\max\left\{2,\frac{2}{x_{1,1}^{2}}\right\} \left(\frac{f\left(-\frac{x}{\sqrt{k}}\right)+\frac{x^{2}}{k}}{f\left(-\sqrt{c}x\right)+\frac{x^{2}}{k}}\right). \label{eq:step1}
	\end{equation}
	We now proceed to bound $R_{ck}=\frac{f\left(-\frac{x}{\sqrt{k}}\right)+\frac{x^{2}}{k}}{f\left(-\sqrt{c}x\right)+\frac{x^{2}}{k}}$. By denoting $z=\frac{x}{\sqrt{k}}$
	we have that
	\begin{equation}
	R_{ck}=\frac{f\left(-z\right)+z^{2}}{f\left(-\sqrt{ck}z\right)+z^{2}}. \label{Rk}
	\end{equation}
	We now compute an upper bound for $R_{ck}$ using the following case analysis: \textbf{(\romannum{1})} $f\left(-z\right)\ge z^{2}$ and $f\left(-\sqrt{ck}z\right)\ge z^{2}$, \textbf{(\romannum{2})} $f\left(-z\right)<z^{2}$ and $f\left(-\sqrt{ck}z\right)<z^{2}$, \textbf{(\romannum{3})}, and \textbf{(\romannum{4})} $f\left(-z\right)<z^{2}$ and $f\left(-\sqrt{ck}z\right)\ge z^{2}$. Let $z_{ck} > 0$ be such that $f\left(-\sqrt{ck}z_{ck}\right)=z_{ck}^{2}$ as given by Lemma~\ref{lemma:sigmoid_intersection}\ref{sig.1}. There are four cases
	
	\textbf{Case (\romannum{1})}: $f\left(-z\right) \ge z^{2}$ and $f\left(-\sqrt{ck}z\right) \ge z^{2}$. Since $f\left(-\sqrt{ck}z\right) \ge z^{2}$, by Lemma~\ref{lemma:sigmoid_intersection}\ref{sig.3} we have that $z \le z_{ck}$. Thus 
	\begin{align}
	f\left(-\sqrt{ck}z\right)\ge & f\left(-\sqrt{ck}z_{ck}\right) \label{eq:26}\\
	= & z_{ck}^{2},\label{eq:27}
	\end{align}
	where~\eqref{eq:26} holds since $f$ is monotonic and $z\le z_{ck}$, and~\eqref{eq:27} is from the definition of $z_{ck}.$ Furthermore,
	by adding $f\left(-z\right)$ to both sides of the assumption $f\left(-z\right)\ge z^{2}$, we have that 
	\begin{equation}
	f\left(-z\right)+z^{2} \le 2f\left(-z\right).\label{eq:74}
	\end{equation}
	Substituting~\eqref{eq:74} and~\eqref{eq:27} in~\eqref{Rk} yields
	\begin{equation}
	R_{ck}=\frac{f\left(-z\right)+z^{2}}{f\left(-\sqrt{ck}z\right)+z^{2}}\le\frac{2f\left(-z\right)}{z_{ck}^{2}}\le\frac{1}{z_{ck}^{2}},\label{eq:case_1_lem_7}
	\end{equation}
	where the last inequality, is since $f(-z) \le 1/2$ for every $z \ge 0$.
	
	\textbf{Case (\romannum{2})}: $f\left(-z\right)<z^{2}$ and $f\left(-\sqrt{ck}z\right)<z^{2}$.
	By adding $z^{2}$ to both sides of the assumption $f\left(-z\right)<z^{2}$, we have that 
	\begin{equation}
	f\left(-z\right)+z^{2}\le2z^{2}.\label{eq:75}
	\end{equation}
	Furthermore, since $f\left(-\sqrt{ck}z\right)>0$ we have that 
	\begin{equation}
	f\left(-\sqrt{ck}z\right)+z^{2}\ge z^{2}.\label{eq:76}
	\end{equation}
	Combining~\eqref{eq:75} and~\eqref{eq:76} yields
	\begin{equation}
	R_{ck}=\frac{f\left(-z\right)+z^{2}}{f\left(-\sqrt{ck}z\right)+z^{2}}\le\frac{2z^{2}}{z^{2}}=2.\label{eq:case_2_lem_7}
	\end{equation}
	
	\textbf{Case (\romannum{3})}: $f\left(-z\right)\ge z^{2}$ and $f\left(-\sqrt{ck}z\right)<z^{2}$. Since $f\left(-\sqrt{ck}z\right)<z^{2}$, by Lemma~\ref{lemma:sigmoid_intersection} we have that $z>z_{ck}$. Thus
	\begin{equation}
	f\left(-\sqrt{ck}z\right)+z^{2}\ge z^{2}\ge z_{ck}^{2}.\label{eq:79}
	\end{equation}
	By adding $f\left(-z\right)$ to both sides of the assumption $f\left(-z\right)\ge z^{2}$, we have that 
	\begin{equation}
	2f\left(-z\right)\ge f\left(-z\right)+z^{2}.\label{eq:80}
	\end{equation}
	Substituting~\eqref{eq:79} and~\eqref{eq:80} in~\eqref{Rk} yields
	\begin{equation}
	R_{ck}=\frac{f\left(-z\right)+z^{2}}{f\left(-\sqrt{ck}z\right)+z^{2}}\le\frac{2f\left(-z\right)}{z_{ck}^{2}}\le\frac{1}{z_{ck}^{2}}.\label{eq:case_3_lem_7}
	\end{equation}
	
	\textbf{Case (\romannum{4})}: $f\left(-z\right)<z^{2}$ and $f\left(-\sqrt{ck}z\right)\ge z^{2}$. By adding $z^{2}$ to both sides of the assumption $f\left(-z\right)<z^{2}$, we have that 
	\begin{equation}
	f\left(-z\right)+z^{2}\le2z^{2}.\label{eq:82}
	\end{equation}
	Since $f\left(-\sqrt{ck}z\right)>0$ we have that 
	\begin{equation}
	f\left(-\sqrt{ck}z\right)+z^{2} > z^{2}.\label{eq:83}
	\end{equation}
	Plugging ~\eqref{eq:82} and~\eqref{eq:83} in~\eqref{Rk} yields
	\begin{equation}
	R_{ck}=\frac{f\left(-z\right)+z^{2}}{f\left(-\sqrt{ck}z\right)+z^{2}}\le\frac{2z^{2}}{z^{2}}=2.\label{eq:case_4_lem_7}
	\end{equation}
	
	Combining the results of the case analysis:~\eqref{eq:case_1_lem_7},~\eqref{eq:case_2_lem_7},~\eqref{eq:case_3_lem_7},and~\eqref{eq:case_4_lem_7} we have that
	\begin{equation}
	R_{ck}\le2+\frac{1}{z_{ck}^{2}}.\label{eq:boundRk}
	\end{equation}
	Furthermore, there exists $k_{0} > 0$ such that for every $k \ge k_{0}$,
	\begin{equation}
	\frac{1}{z_{ck}^{2}}\le ck.\label{eq:86}
	\end{equation}
	Substituting~\eqref{eq:86} in~\eqref{eq:boundRk} yields
	\begin{equation}
	R_{ck}\le 2 + ck,\label{eq:87}
	\end{equation}
	by~\eqref{eq:step1} we have
	\[
	\frac{f\left(cx\right)+\frac{x^{2}}{k}}{f\left(-\sqrt{c}x\right)+\frac{x^{2}}{k}} \le D\max\left\{2,\frac{2}{x_{1,1}^{2}}\right\}R_{ck}.
	\]
	Substituting~\eqref{eq:87} in the last term gives
	\[
	\frac{f\left(cx\right)+\frac{x^{2}}{k}}{f\left(-\sqrt{c}x\right)+\frac{x^{2}}{k}} \le D\max\left\{2,\frac{2}{x_{1,1}^{2}}\right\} \left(2 + ck\right).
	\]
	It holds that for every $k \ge \frac{1}{c}$ we have $2 \le 2ck $ plugging this in the above term yields
	\[
	\frac{f\left(cx\right)+\frac{x^{2}}{k}}{f\left(-\sqrt{c}x\right)+\frac{x^{2}}{k}} \le 3D\max\left\{2,\frac{2}{x_{1,1}^{2}}\right\}ck.
	\]
\end{proof}

The following lemma is similar to Lemma~\ref{lemma:L2RegBoundSig} above, but for the logistic regression function. The proof is similar to the proof of Lemma~\ref{lemma:L2RegBoundSig}.
\begin{lemma}\label{lemma:L2RegBoundLosigtic}
	Let $f$ be the logistic regression function, let $x_{1,1}$ be as in Lemma~\ref{lemma:ratio_simple}, and let $c, R>0$. Assume that there is $D > 1$ such that $\frac{f\left(cy\right)}{f\left(\frac{y}{\sqrt{k}}\right)} \leq D$ for every $y \ge 0$ and $k>0$. Then, there is $k_{0}>0$ such that for every $k\ge k_{0}$ and for every $0 \leq x \leq R$, 
	\[
	\frac{f\left(cx\right)+\frac{x^{2}}{k}}{f\left(-cx\right)+\frac{x^{2}}{k}} \le 3RD\max\left\{2,\frac{2}{x_{1,1}^{2}}\right\}ck.
	\]
\end{lemma}

\begin{lemma} \label{lemma:L2RegBoundSig_main}
	Let $f(x)=\frac{1}{1+e^{-x}}$ for every $x \in \REAL$ and let $c > 0$. Then, there is $k_{0} > 0$ such that for every $k \ge k_{0}$ and for every $x \ge 0$
	\[
	\frac{f\left(cx\right)+\frac{x^{2}}{k}}{f\left(-cx\right)+\frac{x^{2}}{k}} \le 66ck
	\]
\end{lemma}
\begin{proof}
	It holds that $f\left(0\right) > 0$. Applying Lemma~\ref{lemma:sigmoid_intersection} with $k=c=1$ yields $x_{1,1}$ such that $f\left(-x_{1,1}\right) = x_{1,1}^{2}$. We now bound $x_{1,1}$. Calculation shows that
	\[
	f\left(-\sqrt{\ln\left(1.2\right)}\right)>\left(\sqrt{\ln\left(1.2\right)}\right)^{2}.
	\]
	Plugging $x=\sqrt{\ln\left(1.2\right)},k=1,c=1$ in Lemma~\ref{lemma:sigmoid_intersection}\ref{sig.2} yields
	\begin{equation}
	x_{1,1}\ge\sqrt{\ln\left(1.2\right)}.\label{eq:66}
	\end{equation}
	By applying Lemma~\ref{lemma:ratio_simple} with $f$ we have
	\begin{equation} \label{eq:sig_ratio}
	\frac{f\left(x\right)+x^{2}}{f\left(-x\right)+x^{2}} \le \max\left\{2,\frac{2}{x_{1,1}^{2}}\right\} \le11,
	\end{equation}
	where the last inequality is by~\eqref{eq:66}.
	
	Since $f\left(y\right) \le 1$ for every $y >0$ and $f\left(\frac{x}{\sqrt{k}}\right) \ge \frac{1}{2}$, for every $c,k > 0$ we have that
	\begin{equation} \label{eq:123}
	\frac{f\left(cx\right)}{f\left(\frac{x}{\sqrt{k}}\right)} \le 2.
	\end{equation}
	Applying Lemma~\ref{lemma:L2RegBoundSig} with $f,D=2$ yields
	\begin{equation}
	\frac{f\left(cx\right)+\frac{x^{2}}{k}}{f\left(-cx\right)+\frac{x^{2}}{k}} \le 66kc.
	\end{equation}
\end{proof}

\begin{lemma}\label{lemma:logistic}
	Let $f=\log(1+e^x)$ for every $x \in \REAL$ and let $c, R > 0$. Then, there is $k_{0} > 0$ such that for every $k \ge k_{0}$ and for every $0 \le x \le R$
	\[
	\frac{f(cx)+\frac{x^{2}}{k}}{f(-cx)+\frac{x^2}{k}} \le 3R\frac{\log\left(2e^{cR}\right)}{\log(2)}kc.
	\]
\end{lemma}
\begin{proof}
	Let $0 \le x \le R$. Applying Lemma~\ref{lemma:sigmoid_intersection} with $k=c=1$ yields $x_{1,1}$ such that $f\left(-x_{1,1}\right) = x_{1,1}^{2}$. We now bound $x_{1,1}$. By simple calculations we have that
	\[
	f\left(-\sqrt{\ln\left(1.2\right)}\right)>\left(\sqrt{\ln\left(1.2\right)}\right)^{2}.
	\]
	Plugging $x=\sqrt{\ln\left(1.2\right)},k=1,c=1$ in Lemma~\ref{lemma:sigmoid_intersection}\ref{sig.2} yields
	\begin{equation}
	x_{1,1}\ge\sqrt{\ln\left(1.2\right)}.\label{eq:66}
	\end{equation}
	
	For every $c,k > 0$, since $x \le R$ and $f$ is non-decreasing we have that
	\begin{equation}
	f\left(cx\right) \le f\left(cR\right), \label{eq:up}
	\end{equation}
	furthermore, since $x \ge 0$ and $f$ is increasing we have that
	\begin{equation}
	f\left(\frac{x}{\sqrt{k}}\right) \ge \log\left(2\right). \label{eq:down}
	\end{equation}
	We have that
	\begin{align}
	\frac{f\left(cx\right)}{f\left(\frac{x}{\sqrt{k}}\right)} & \le \frac{f\left(cR\right)}{\log\left(2\right)}\label{eq:e} \\ 
	& = \frac{\log\left(1 + e^{cR}\right)}{\log\left(2\right)}\label{eq:c} \\ 
	& \le \frac{\log\left(2e^{cR}\right)}{\log\left(2\right)},\label{eq:d}
	\end{align}
	where~\eqref{eq:e} is by~\eqref{eq:up} and~\eqref{eq:down},~\eqref{eq:c} is by the definition of $f$ and ~\eqref{eq:d} holds since $Rc > 0$.
	Applying Lemma~\ref{lemma:L2RegBoundLosigtic} with $f,D=\frac{\log\left(2e^{cR}\right)}{\log\left(2\right)}$ yields
	\begin{equation}
	\frac{f\left(cx\right)+\frac{x^{2}}{k}}{f\left(-cx\right)+\frac{x^{2}}{k}} \le 3R\frac{\log\left(2e^{cR}\right)}{\log\left(2\right)}kc.
	\end{equation}
\end{proof}

\subsection{Proofs of Our Main Claims} \label{sec:mainProofs}

We start by proving the main claims with respect to the sigmoid activation function; see Lemma~\ref{lem11} and Theorem~\ref{theorem:sigmoid_appendix}. We then prove the main claims for the logistic regression activation function; see Lemma~\ref{lem7} and Theorem~\ref{theorem:logistic_appendix}.

\renewcommand{\ck}{c_{\mathrm{sigmoid},k}}
\begin{lemma}\label{lem11}
	Let $P=\left\{ \b{p}_{1},\ldots,\b{p}_{n}\right\} \subset\REAL^{d}$ be a set of points, sorted by their length. I.e. $\norm{\b{p}_{i}} \le \norm{\b{p}_{j}}$
	for every $1\le i\le j\le n$. Let $k>0$ be a sufficiently large constant and $\ck(\b{p},\b{x})=\frac{1}{1+e^{-\b{p}\cdot\b{x}}}+\frac{\norm{\b{x}}^2}{k}$
	for every $\b{x}\in\REAL^{d}$ and $\b{p}\in P$.
	Then the sensitivity of every $p_j\in P$ is bounded by
$s(\b{p})=s_{P,\mathbf{1},\REAL^d,\ck}(\b{p})\in O\left(\frac{\norm{\b{p}_j}k+1}{j}\right)$,
and the total sensitivity is
\[
t=\sum_{\b{p}\in P}s(\b{p})\in O\left(\log n +k\sum_{j=1}^n \frac{\norm{\b{p}_j}}{j}\right).
\]
\end{lemma}
\ifproofs
\begin{proof}
	Define $f(z)=\frac{1}{1+e^{-z}}$ and $g(z)=z^{2}$
	for every $z\in\REAL$. Let $\b{x}\in\REAL^{d}$, $\b{p}_{j}\in P$ and $i\in[1,j]$ be an integer. We substitute $c=\norm{\b{p}_{i}}$ in Lemma~\ref{lemma:L2RegBoundSig_main} to
	obtain that for every $z>0$
	\[
	\frac{f(\norm{\b{p}_{i}}z)+\frac{z^{2}}{k}}{f(-\norm{\b{p}_{i}}z)+\frac{z^{2}}{k}}\le 66 \norm{\b{p}_{i}}k.
	\]
	Denote $b_{\b{p}_{i}}=66\norm{\b{p}_{i}}k$
	and multiply the above term by $f\left(-\left\Vert \b{p}_{i}\right\Vert z\right)+\frac{z^{2}}{k}$
	to get
	\[
	f\left(\left\Vert \b{p}_{i}\right\Vert z\right)+\frac{z^{2}}{k}\le b_{\b{p}_{i}}\left(f\left(-\left\Vert \b{p}_{i}\right\Vert z\right)+\frac{z^{2}}{k}\right).
	\]
	Substituting in Lemma~\ref{lemma:exsitance of inf bounded}
$\b{p}=\b{p}_{i},f\left(z\right)=\frac{1}{1+e^{-z}},g\left(z\right)=z^{2},M=1,f\left(0\right)=\frac{1}{2}$
	yields
	\begin{equation}
	\max_{\b{p'}\in P}\ck\left(\b{p'},\b{x}\right)\le2\left(b_{\b{p}_{i}}+1\right)\ck\left(\b{p}_{i},\b{x}\right).\label{eq:13-1}
	\end{equation}
	Thus
	\begin{align}
	\ck\left(\b{p}_{j},\b{x}\right)\le & \max_{\b{p'}\in P}\ck\left(\b{p'},\b{x}\right)\label{eq:14}\\
	\le & 2\left(b_{\b{p}_{i}}+1\right)\ck\left(\b{p}_{i},\b{x}\right),\label{eq:15}
	\end{align}
	where~\eqref{eq:14} is since $\b{p}_{j}\in P$ and~\eqref{eq:15}
	is by~\eqref{eq:13-1}. Dividing both sides by $2\left(b_{\b{p}_{i}}+1\right)$
	yields
	\begin{equation}
	\ck\left(\b{p}_{i},\b{x}\right)\ge\frac{\ck\left(\b{p}_{j},\b{x}\right)}{2\left(b_{\b{p}_{i}}+1\right)}.\label{eq:upper}
	\end{equation}
	
	We now proceed to bound the sensitivity of $\b{p}_{j}$.
	Since the set of points $\left\{ \b{p}_{1},\ldots,\b{p}_{j}\right\} $
	is a subset of $P$, and since the cost function $\ck\left(\b{p}_{j},\b{x}\right)$
	is positive we have that
	\begin{equation}
	\sum_{\b{p'}\in P}\ck\left(\b{p'},\b{x}\right)\ge\sum_{i=1}^{j}\ck\left(\b{p}_{i},\b{x}\right).\label{eq:17-1}
	\end{equation}
	By summing~\eqref{eq:upper} over $i\leq j$, we obtain
	\begin{equation}\label{eq:18-1}
\begin{split}
	\sum_{i=1}^{j}\ck
\left(\b{p}_{i},\b{x}\right)
&\ge \ck(\b{p}_{j},\b{x})\sum_{i=1}^{j}\frac{1}{2(b_{\b{p}_{i}}+1)}\\
&\geq \ck(\b{p}_{j},\b{x})\frac{j}{2(b_{\b{p}_{j}}+1)},
\end{split}
	\end{equation}
where the last inequality holds since $b_{p_i}=66\norm{p_i}k\leq b_{\b{p}_j}$ for every $i\leq j$.
	Combining~\eqref{eq:17-1} and~\eqref{eq:18-1} yields
	\begin{equation}
	\sum_{\b{p'}\in P}\ck(\b{p'},\b{x})
\ge \frac{j\ck(\b{p}_{j},\b{x})}{2 (b_{\b{p}_{j}}+1)}\label{eq:19-1}
	\end{equation}
	Therefore, the sensitivity is bounded by
	\[
\begin{split}
	s_{P,\b{1},\REAL^d, \ck}(\b{p}_{j})&=  \sup_{\b{x}\in \REAL^d}\frac{\ck(\b{p}_{j},\b{x})}{\sum_{\b{p'}\in P}\ck(\b{p'},\b{x})}\\
&\leq\frac{2(b_{p_j}+1)}{j}\leq \frac{2(66\norm{p_j}k+1)}{j}.
\end{split}
\]
Summing this sensitivity bounds the total sensitivity by
\[
\sum_{j=1}^n\frac{2(66\norm{p_j}k+1)}{j}\in O\left(\log n+k\sum_{j=1}^n\frac{\norm{p_j}}{j}\right).
\]

\end{proof}\fi

In what follows is the main claim and proof for the sigmoid activation function.
\begin{theorem} [Theorem~\ref{theorem:sigmoid}] \label{theorem:sigmoid_appendix}
Let $P$ be a set of $n$ points in the unit ball of $\REAL^d$, $\eps,\delta\in(0,1)$, and $k>0$ be a sufficiently large constant.
For every $p,x\in\REAL^d$, let $\ck\left(\b{p},\b{x}\right)=\frac{1}{1+e^{-\b{p}\cdot\b{x}}}+\frac{\left\Vert \b{x}\right\Vert ^{2}}{k}$.
Finally, let $(Q,u)$ be the output of a call to $\alg(P,\eps,\delta,k)$; see Algorithm~\ref{algno}.
Then, with probability at least $1-\delta$, $(Q,u)$ is an $\eps$-coreset for $(P,\mathbf{1},\REAL^d,\ck)$.
Moreover, for $t=(1+k)\log n$ we have $|Q|\in O\left(\frac{t}{\varepsilon^{2}}\left(d\log t+\log\frac{1}{\delta}\right)\right)$, and $(Q,u)$ can be computed in $O(dn+n\log n)$ time.
\end{theorem}
\ifproofs
\begin{proof}
By~\cite{huggins2016coresetsLogistic}, the dimension of $(P,w,\REAL^d,c)$ is at most $d+1$, where $(P,w)$ is a weighted set, $P\subseteq \REAL^d$, and $c(p,x)=f\left(\b{p}\cdot\b{x}\right)$ for some monotonic and invertible function $f$.
By Lemma~\ref{lem11}, the total sensitivity of $(P,\mathbf{1},\REAL^d,\ck)$ is bounded by
\[
\begin{split}
t\in O\left(\log n +k\sum_{j=1}^n \frac{\norm{p_j}}{j}\right)
&=O\left(\log n +k\sum_{j=1}^n \frac{1}{j}\right)\\
\>=O\left((1+k)\log n\right),
\end{split}
\]
where the last equality holds since the input points are in the unit ball.

Plugging these upper bounds on the dimension and total sensitivity of the query space in Theorem~\ref{theorem:sens_is_coreset}, yields that a call to Algorithm~\ref{algno}, which samples points from $P$ based on their sensitivity bound, returns the desired coreset $(Q,u)$. The running time is dominated by sorting the length of the points in $O(n\log n)$ time after computing them in $O(nd)$ time. 
\end{proof}\fi


\renewcommand{\ck}{c_{\mathrm{logistic},k}}
\begin{lemma}\label{lem7}
	Let $P=\left\{ \b{p}_{1},\ldots,\b{p}_{n}\right\} \subset\REAL^{d}$ be a set of points, sorted by their length, i.e. $\norm{\b{p}_{i}} \le \norm{\b{p}_{j}}$
	for every $1\leq i\leq j\leq n$. Let $R>0$, $k>0$ be a sufficiently large constant and $\ck(\b{p},\b{x})=\log(1+e^{\b{p}\cdot \b{x}})+\frac{\norm{\b{x}}^2}{k}$
	for every $\b{x}\in B(\b{0}, R)$ and $\b{p}\in P$. Denote by $B(\b{0},R)$ the ball of radius $R$ centered at the origin.
	Then the sensitivity of every $p_j\in P$ is bounded by
	$s(\b{p})=s_{P,\mathbf{1},B(\b{0}, R),\ck}(\b{p})\in O\left(\frac{R^3\norm{\b{p}_j}k + R^2}{j}\right)$,
	and the total sensitivity is
	\[
	t=\sum_{p\in P}s(p)\in O\left(R^2\log n+R^3k\sum_{j=1}^n\frac{\norm{p_j}}{j}\right).
	\]
\end{lemma}
\ifproofs
\begin{proof}
	Define $f(z)=\log(1+e^z)$ and $g(z)=z^{2}$
	for every $z\in\REAL$. Let $\b{x}\in\REAL^{d}$, $\b{p}_{j}\in P$ and $i\in[1,j]$ be an integer. We substitute $c=\norm{\b{p}_{i}}$ in Lemma~\ref{lemma:logistic} to
	obtain that for every $z>0$
	\[
	\frac{f(\norm{\b{p}_i}x)+\frac{x^{2}}{k}}{f(-\norm{\b{p}_i}x)+\frac{x^2}{k}} \le 3R\frac{\log\left(2e^{\norm{\b{p}_i}R}\right)}{\log(2)}k\norm{\b{p}_i}.
	\]
	Denote $b_{\b{p}_{i}}=3R\frac{\log\left(2e^{\norm{\b{p}_i}R}\right)}{\log(2)}k\norm{\b{p}_i}$
	and multiply the above term by $f(-\norm{\b{p}_{i}}z)+\frac{z^{2}}{k}$
	to get
	\[
	f(\norm{\b{p}_{i}}z)+\frac{z^{2}}{k}\le b_{\b{p}_{i}}\left(f(-\norm{\b{p}_{i}}z)+\frac{z^{2}}{k}\right).
	\]
	Substituting in Lemma~\ref{lemma:exsitance of inf bounded_appendix}
	$\b{p}=\b{p}_{i},f(z)=\log(1+e^z),g(z)=z^{2},M=\log(1+e^R),f(0)=\log(2)$
	yields
	\begin{equation} \label{eq:13}
	\begin{split}
	&\max_{\b{p'}\in P}\ck(\b{p'},\b{x})\le \\ &\frac{\log(1+e^R)(b_{\b{p}_{i}}+1)\ck\left(\b{p}_{i},\b{x}\right)}{log(2)}.
	\end{split}
	\end{equation}
	Thus
	\begin{align}
	&\ck(\b{p}_{j},\b{x})\le  \max_{\b{p'}\in P}\ck(\b{p'},\b{x})\label{eq:14} \le \\
	&\frac{\log(1+e^R)(b_{\b{p}_{i}}+1)\ck\left(\b{p}_{i},\b{x}\right)}{log(2)},\label{eq:15}
	\end{align}
	where~\eqref{eq:14} is since $\b{p}_{j}\in P$ and~\eqref{eq:15}
	is by~\eqref{eq:13}. Dividing both sides by $\frac{\log(1+e^R)}{log(2)}(b_{\b{p}_{i}}+1)$
	yields
	\begin{equation}
	\ck(\b{p}_{i},\b{x})\ge\frac{\ck(\b{p}_{j},\b{x})}{\frac{\log(1+e^R)}{log(2)}(b_{\b{p}_{i}} + 1)}.\label{eq:upper}
	\end{equation}
	
	We now proceed to bound the sensitivity of $\b{p}_{j}$.
	Since the set of points $\left\{ \b{p}_{1},\ldots,\b{p}_{j}\right\} $
	is a subset of $P$, and since the cost function $\ck(\b{p}_{j},\b{x})$
	is positive we have that
	\begin{equation}
	\sum_{\b{p'}\in P}\ck(\b{p'},\b{x})\ge\sum_{i=1}^{j}\ck(\b{p}_{i},\b{x}).\label{eq:17-1}
	\end{equation}
	By summing~\eqref{eq:upper} over $i\leq j$, we obtain
	\begin{equation}\label{eq:18-1}
	\begin{split}
	&\sum_{i=1}^{j}\ck\b{p}_{i},\b{x}) \ge\\
	&\ck(\b{p}_{j},\b{x})\sum_{i=1}^{j}\frac{\log(2)}{\log(1+e^R)(b_{\b{p}_{i}} + 1)}\ge\\
	&\ck(\b{p}_{j},\b{x})\frac{j\log(2)}{\log(1+e^R)(b_{\b{p}_{j}} + 1)},
	\end{split}
	\end{equation}
	where the last inequality holds since $b_{p_i}=3R\frac{\log(1+e^R)}{log(2)}\norm{p_i}k\leq b_{\b{p}_j}$ for every $i\leq j$.
	Combining~\eqref{eq:17-1} and~\eqref{eq:18-1} yields
	\begin{equation}
	\sum_{\b{p'}\in P}\ck(\b{p'},\b{x}) \ge \frac{j\log(2)\ck(\b{p}_{j},\b{x})}{\log(1+e^R)(b_{\b{p}_{j}} + 1)}\label{eq:19-1}
	\end{equation}
	Therefore, the sensitivity is bounded by
	\[
	\begin{split}
	& s_{P,\b{1},B(\b{0},R), \ck}(\b{p}_{j}) =\\
	& \sup_{\b{x}\in B(\b{0}, R)}\frac{\ck(\b{p}_{j},\b{x})}{\sum_{\b{p'}\in P}\ck(\b{p'},\b{x})} \leq\\
	& \frac{\log(1+e^R)(b_{\b{p}_{j}} + 1)}{j\log(2)} \leq \\ 
	& \frac{\log(1+e^R)\left(\frac{3R\log(1+e^R)}{log(2)}\norm{p_j}k + 1\right)}{j\log(2)}.
	\end{split}
	\]
	Thus, $s_{P,\b{1},B(\b{0},R), \ck}(\b{p}_{j}) \in
	O\left(\frac{R^3\norm{\b{p}_j}k + R^2}{j}\right)$.
	Summing this sensitivity bounds the total sensitivity by
	\[
	\sum_{j=1}^n\frac{R^3\norm{\b{p}_j}k + R^2}{j}\in O\left(R^2\log n+R^3k\sum_{j=1}^n\frac{\norm{p_j}}{j}\right).
	\]
\end{proof}\fi

In what follows is the main claim and proof for the logistic regression function.
\begin{theorem}[Theorem~\ref{theorem:logistic}] \label{theorem:logistic_appendix}
	Let $P$ be a set of $n$ points in the unit ball of $\REAL^d$, $\eps,\delta\in(0,1)$, and $R,k>0$ where $k$ is a sufficiently large constant.
	For every $\b{p}\in\REAL^d$,$\b{x}\in B(\b{0}, R)$ let $
	\ck(\b{p},\b{x})=\log\left(1+e^{\b{p}\cdot\b{x}}\right)+\frac{\norm{\b{x}}^{2}}{k}$.
	Let $(Q,u)$ be the output of a call to $\alg(P,\eps,\delta,k)$.
	Moreover, for $t=R\log n(1+Rk)$, we have that $|Q|\in O\left(\frac{t}{\varepsilon^{2}}\left(d^2\log t+\log\frac{1}{\delta}\right)\right)$ and $(Q,u)$ can be computed in $O(dn+n\log n)$ time.
\end{theorem}
\ifproofs
\begin{proof}
	By~\cite{huggins2016coresetsLogistic}, the dimension of $(P,w,\REAL^d,c)$ is at most $d+1$, where $(P,w)$ is a weighted set, $P\subseteq \REAL^d$, and $c(p,x)=f\left(\b{p}\cdot\b{x}\right)$ for some monotonic and invertible function $f$.
	By Lemma~\ref{lem7}, the total sensitivity of $(P,\mathbf{1},\REAL^d,\ck)$ is bounded by
	\[
	\begin{split}
	&t\in O\left(R^2\log n+R^3k\sum_{j=1}^n\frac{\norm{p_j}}{j}\right) = \\
	&O\left(R^2\log n +R^3k\sum_{j=1}^n \frac{1}{j}\right)
	=O\left(R^2\log n(1+Rk)\right),
	\end{split}
	\]
	where the last equality holds since the input points are in the unit ball.
	
	Plugging these upper bounds on the dimension and total sensitivity of the query space in Theorem~\ref{theorem:sens_is_coreset}, yields that a call to \alg, which samples points from $P$ based on their sensitivity bound, returns the desired coreset $(Q,u)$. The running time is dominated by sorting the length of the points in $O(n\log n)$ time after computing them in $O(nd)$ time. Sampling $m=|Q|$ points from $n$ points according to such a given distribution takes $O(1)$ time after pre-processing of $O(n)$ time.
\end{proof}\fi

\section{Bounding the VC-dimension} \label{sec:VCBound}
In what follows we first give the formal definition of the VC dimension of a given query space. We then formally bound the VC dimensions of the sigmoid and logistic regression cost functions.

\begin{definition}[VC-dimension] \label{def:VCdim}
\cite{feldman2011unified,Vap71a} For a query space $\left(P,w,X,c\right)$
we define
\[
range\left(\b{x},r\right)=\left\{ \b{p}\in P\mid w\left(\b{p}\right)c\left(\b{p}.\b{x}\right)\le r\right\},
\]
for every $x\in X$ and $r\geq 0$ . The (VC) dimension of $\left(P,w,X,c\right)$
is the size $|G|$ of the largest subset $G\subseteq P$ such that
have
\[
\left|\left\{ G\cap range\left(\b{x},r\right)|\b{x}\in X,r\geq 0\right\} \right|=2^{|G|}.
\]
\end{definition}


\begin{theorem}[Theorem 8.14 in~\cite{lucic2017training} and generalized in~\cite{lucic2017training}] \label{theorem:VCbound}
Let $h$ be a function from $\REAL^m \times \REAL^d$ to $\br{0,1}$, determining the class
\[
\mathcal{H} = \br{h_\theta(\cdot) \mid h_\theta:\mathcal{X} \to \REAL_{++}, \theta \in \REAL^m}.
\]
Suppose that $h$ can be computed by an algorithm that takes as input the pair $(\theta, x) \in \REAL^m \times \REAL^n$ and returns $h_\theta(x)$ after no more than $t$ of the following operations:
\begin{enumerate}
    \item the arithmetic operations $+,-,\times$ and $/$ on real numbers,
    \item jumps conditioned on $>, \geq, <, \leq,=$ and $\neq$ comparisons of real numbers, and
    \item output $0, 1$
\end{enumerate}
and no more than $p$ operations of the exponential function $x\to e^x$ on real numbers, then the VC-dimension of $\mathcal{H}$ is $O(m^2p^2 + mp(t+\log(mp)))$.
\end{theorem}

\newcommand{\cs}{c_{\mathrm{sigmoid},k}}
\begin{lemma} [VC-dimension of the Sigmoid loss function]
Let $k>0$ be a constant and $(P,w,\REAL^d,\cs)$ be a query space where $\cs(\b{p},\b{x})=\frac{1}{1+e^{-\b{p}\cdot\b{x}}}+\frac{\norm{\b{x}}^2}{k}$ for every $\b{x}\in\REAL^{d}$ and $\b{p}\in P$. Then the VC-dimension of $(P,w,\REAL^d,\cs)$ is at most $O(d^2)$.
\end{lemma}
\begin{proof}
Observe that for every $\b{x}\in\REAL^{d}$ and $\b{p}\in P$ we can evaluate $\cs(\b{p},\b{x})$ using $t\in O(d)$ addition, multiplication, and division operations and $p \in O(1)$ operations of the exponential function $x \to e^x$. Then, by Theorem~\ref{theorem:VCbound}, the VC-dimension of $(P,w,\REAL^d\cs)$ is bounded by $O(d^2)$.
\end{proof}

\begin{lemma} [VC-dimension of the Logistic Regression loss function]
Let $k>0$ be constants and $(P,\b{1},\REAL^d,\ck)$ be a query space 
where $\ck(\b{p},\b{x})=\log(1+e^{\b{p}\cdot \b{x}})+\frac{\norm{\b{x}}^2}{k}$ for every $\b{x}\in \REAL^d$ and $\b{p}\in P$. Then the VC-dimension of $(P,\b{1},\REAL^d,c)$ is at most $O(d^2)$.
\end{lemma}
\begin{proof}
We first bound the VC-dimension of $(P,\b{1},\REAL^d,g)$, where
\[
g(\b{p},\b{x}) = e^{c(\b{p},\b{x})} = (1+e^{\b{p}\cdot \b{x}}) e^{\frac{\norm{\b{x}}^2}{k}}.
\]
Observe that we can evaluate $g(\b{p},\b{x})$ using $t\in O(d)$ addition, multiplication, and division operations and $p \in O(1)$ operations of the exponential function $x \to e^x$. Then, by Theorem~\ref{theorem:VCbound}, the VC-dimension of $(P,\b{1},\REAL^d,g)$ is bounded by $O(d^2)$.

We now show that the VC-dimension of $(P,\b{1},\REAL^d,c)$ is upper bounded by the VC-dimension of $(P,\b{1},\REAL^d,g)$.
Recall that for the query space $(P,\b{1},\REAL^d,c)$ and every $\b{x} \in \REAL^d$ and $r \geq 0$ we have that $range(\b{x},r) = \br{\b{p} \in P \mid c(\b{p},\b{x}) \leq r}$. For the query space $(P,\b{1},\REAL^d,g)$ and every $\b{x} \in \REAL^d$ and $r \geq 0$ we have that $range'(\b{x},r) = \br{\b{p} \in P \mid g(\b{p},\b{x}) \leq r}$.

For every $r \geq 0$ let $r_g := e^{r}$. Then we have that
\[
\begin{split}
range(\b{x},r) & = \br{\b{p} \in P \mid c(\b{p},\b{x}) \leq r} 
= \br{\b{p} \in P \mid e^{c(\b{p},\b{x})}\leq e^{r}} \\
& = \br{\b{p} \in P \mid g(\b{p},\b{x}) \leq r_g} 
= range'(\b{x},r_g).
\end{split}
\]

Therefore, for every $G\subseteq P$ we have that
\[
\left|\left\{ G\cap range\left(\b{x},r\right)|\b{x}\in X,r\geq 0\right\} \right| \leq \left|\left\{ G\cap range'\left(\b{x},r_g\right)|\b{x}\in X,r_g\geq 0\right\} \right|.
\]
Hence, by the definition of VC-dimension (see Definition~\ref{def:VCdim}), we have that the VC-dimension of the query space $(P,\b{1},\REAL^d,c)$ is upper bounded by the VC-dimension of the query space $(P,\b{1},\REAL^d,g)$ which is upper bounded by $O(d^2)$.
\end{proof}

\section{Known results}
For completeness, in what follows we formally state known claims, which were utilized in the proofs of the previous sections.

\begin{theorem}[Intermediate Value Theorem]\label{theorem:Intermediate-Value-Theorem}
	Let $a,b\in\REAL$ such that $a<b$ and let $f:\left[a,b\right]\to\REAL$ be a continuous function. Then for every $u$ such that 
	\[
	\min\left\{ f\left(a\right),f\left(b\right)\right\} \le u\le\max\left\{ f\left(a\right),f\left(b\right)\right\} ,
	\]
	there is $c\in\left(a,b\right)$ such that $f\left(c\right)=u.$
\end{theorem}
\begin{theorem}[Mean Value Theorem]\label{theorem:Mean-Value-Theorem} 
	Let $a,b\in\REAL$ such that $a<b$ and $f:\left[a,b\right]\to\REAL$ a continuous
	function on the closed interval $\left[a,b\right]$ and differentiable
	on the open interval $\left(a,b\right).$ Then there is $c\in\left(a,b\right)$
	such that 
	\[
	f'\left(c\right)=\frac{f\left(b\right)-f\left(a\right)}{b-a}.
	\]
\end{theorem}
\begin{theorem}[Inverse of Strictly Monotone Function Theorem] \label{theorem:Inverse-of-Strictly}
	Let $I\subseteq\REAL.$ Let $f:I\to\REAL$ be strictly monotonic function. Let the image
	of $f$ be $J$. Then $f$ has an inverse function $f^{-1}$and
\end{theorem}
\begin{itemize}
	\item If $f$ is strictly increasing then so is $f^{-1}$.
	\item If $f$ is strictly decreasing then so is $f^{-1}$.
\end{itemize}

\end{document}